\documentclass[preprint,11pt]{elsarticle}
\usepackage[english]{babel}
\usepackage{bbm}
\usepackage{amsmath,amsthm,amssymb}
\usepackage{xcolor}
\usepackage{color,soul}
\usepackage{natbib}
\usepackage{mathtools}
\usepackage[hidelinks]{hyperref}
\usepackage{float} 
\usepackage[super]{nth}
\usepackage{cancel}
\usepackage{multirow}
\usepackage{array}
\usepackage{xfrac}
\usepackage{tabularx}
\usepackage{caption}
\usepackage{subcaption}
\usepackage{setspace}
\usepackage[acronyms,toc,symbols]{glossaries-extra}
\setabbreviationstyle[acronym]{long-short}

\usepackage[ruled,vlined]{algorithm2e}
\usepackage{nicefrac}

\usepackage{ifthen}
\let\oldcite=\cite
\renewcommand\cite[1]{\ifthenelse{\equal{#1}{NEEDED}}{[citation~needed]}{\oldcite{#1}}}

\theoremstyle{definition}

\newtheorem{lemma}{Lemma}
\newtheorem{proposition}{Proposition}

\allowdisplaybreaks
\bibpunct[, ]{(}{)}{,}{a}{}{,}%
\def\newblock{\ }%

\newacronym{MILP}{MILP}{mixed integer linear programming}
\newacronym{MIP}{MIP}{mixed integer program}
\newacronym{LP}{LP}{linear program}
\newacronym{SOCP}{SOCP}{second-order cone program}
\newacronym{BFS}{BFS}{Bilevel Feature Selection}
\newacronym{BFS-CV}{BFS-CV}{Bilevel Feature Selection with cross-validation}
\newacronym{ERM}{ERM}{Empirical Risk Minimization}
\newacronym{SAA}{SAA}{sample average approximation}
\newacronym{SVR}{SVR}{support vector regression}
\newacronym{NLP}{NLP}{nonlinear program}
\newacronym{KKT}{KKT}{Karush-Kuhn-Tucker}
\newacronym{BHO}{BHO}{Bilevel Hyperparameter Optimization}
\newacronym{SQP}{SQP}{Sequential Quadratic Programming}

\begin{document}

\linespread{1.6}\selectfont

\begin{center}
	
\vspace*{-0.5cm}

\begin{huge}
	Bilevel Optimization for Feature Selection in the Data-Driven Newsvendor Problem
\end{huge}

\vspace*{0.6cm}

{\setstretch{1.0}

\textbf{Breno Serrano}$^*$\\
{\footnotesize School of Management, Technical University of Munich, Germany, breno.serrano@tum.de}\\[10pt]

\textbf{Stefan Minner}\\
{\footnotesize School of Management and Munich Data Science Institute, Technical University of Munich, Germany, stefan.minner@tum.de}\\[10pt]

\textbf{Maximilian Schiffer}\\
{\footnotesize School of Management and Munich Data Science Institute, Technical University of Munich, Germany, schiffer@tum.de}\\[10pt]

\textbf{Thibaut Vidal}\\[2pt]
{\footnotesize Department of Mathematical and Industrial Engineering, Polytechnique Montréal, Canada,\\
Department of Computer Science, Pontifical Catholic University of Rio de Janeiro, Brazil, thibaut.vidal@polymtl.ca}\\
}
\end{center}

\noindent
\textbf{Abstract.}
We study the feature-based newsvendor problem, in which a decision-maker has access to historical data consisting of demand observations and exogenous features. In this setting, we investigate feature selection, aiming to derive sparse, explainable models with improved out-of-sample performance. Up to now, state-of-the-art methods utilize regularization, which penalizes the number of selected features or the norm of the solution vector. As an alternative, we introduce a novel bilevel programming formulation. The upper-level problem selects a subset of features that minimizes an estimate of the out-of-sample cost of ordering decisions based on a held-out validation set. The lower-level problem learns the optimal coefficients of the decision function on a training set, using only the features selected by the upper-level. We present a mixed integer linear program reformulation for the bilevel program, which can be solved to optimality with standard optimization solvers. Our computational experiments show that the method accurately recovers ground-truth features already for instances with a sample size of a few hundred observations. In contrast, regularization-based techniques often fail at feature recovery or require thousands of observations to obtain similar accuracy. Regarding out-of-sample generalization, we achieve improved or comparable cost performance.

\vspace*{0.2cm}

\vspace*{-0.5cm}

\noindent
$^*$ Corresponding author

\noindent
Declarations of interest: none

\linespread{2.0}\selectfont
\setlength{\abovedisplayskip}{3pt}
\setlength{\belowdisplayskip}{3pt}

\section{Introduction}


The newsvendor problem and its variants have served as fundamental building blocks for models in inventory and supply chain management. In the classical newsvendor problem, a decision-maker optimizes the inventory of a perishable product that has a stochastic demand with a known distribution. However, having complete knowledge of the demand distribution is a strong assumption that does not hold in practice: often, the only information available is a limited set of historical data. Against this background, data-driven approaches became popular and strive to use past demand data to inform the newsvendor's ordering decisions.


In this context, we study the feature-based newsvendor problem~(cf. \citealt{beutel2012safety,ban2019big}) in which the decision-maker has access not only to historical demand observations but also to a set of feature variables---often referred to as contextual information or covariates---that may provide partial information about future realizations of the uncertain demand. 
Companies nowadays have large amounts of data that are used to train machine learning models with the aim of improving operational decisions. In practice, such models often suffer from overfitting to the training data, or lack explainability, which is crucial, e.g., when dealing with high-stakes decisions. In this setting, selecting a subset of the available features can lead to sparser, more explainable models with improved out-of-sample performance. Against this background, we investigate the challenge of feature selection~(cf. \citealt{1183917,kuhn2019feature}): given a data set with a possibly large set of feature variables, we aim to learn a linear decision function for the feature-based newsvendor that can generalize to out-of-sample data, utilizing only relevant features.


The goal of this paper is to propose an approach to feature selection based on bilevel optimization. We search for a subset of features that leads to a minimal out-of-sample cost measured on a held-out data set, when used for training a linear decision function for the feature-based newsvendor. In the remainder of this section, we first review related literature before we detail our contribution and describe the organization of this paper.

\subsection{Related Works}\label{section-literature}

Our work relates to the fields of data-driven optimization for the newsvendor problem, and more broadly to prescriptive analytics, machine learning, and bilevel programming. We briefly review the most related papers in the following.

\vspace*{0.5cm}


\noindent\textbf{Newsvendor problem.} Research on the newsvendor problem often assumed a decision-maker with full knowledge about the demand distribution, and considered various settings, e.g., with different objectives or utility functions~\citep{chen2007risk,wang2009loss}, pricing policies~\citep{petruzzi1999pricing}, and multi-product or multi-period settings~\citep{lau1996newsstand,kogan2003multi}. For general surveys on newsvendor models and extensions, we refer the interested reader to~\citet{khouja1999single}, \citet{QIN2011361} and \citet{choi2012handbook}. In practice, the decision-maker often has only a finite set of demand observations and cannot estimate the true underlying distribution, which motivated works on the distribution-free newsvendor problem. In this context, the seminal work of \mbox{\citet{scarf1958min}} derived the optimal order quantity that maximizes profit against the worst-case demand distribution, assuming that only the mean and variance of demand are known. For a review on the distribution-free newsvendor and extensions thereof, we refer to \citet{gallego1993distribution,moon1994distribution}, and \citet{yue2006expected}. Later works on this problem variant assumed additional information about the demand distribution, such as percentiles~\citep{gallego2001minimax}, symmetry, and unimodality~\citep{perakis2008regret}.


In contrast to working with moments or distributional parameters, data-driven approaches build directly upon a sample of available data that reflects realizations of the underlying uncertainty. In this context, a common solution approach is \gls{SAA}~(cf.~\citealt{kleywegt2002sample,SHAPIRO2003353}). \citet{levi2007provably} applied \gls{SAA} for the single-period featureless newsvendor problem and established upper bounds on the number of samples required to achieve a specified relative error. In this course, \citet{levi2015data}, \citet{cheung2019sampling}, and \citet{besbes2021big} further improved upon previous \gls{SAA} bounds. \citet{besbes2013implications}, \cite{SACHS201428}, and \citet{ban2020confidence} studied the impact of demand censoring, i.e., a problem variant in which only \textit{sales} observations are available but excess demand is not recorded. They derived upper and lower bounds on the difference between the cost achieved by a policy and the optimal cost with knowledge of the demand distribution. Adopting a robust optimization perspective, \citet{bertsimas2005data} proposed a data-driven approach that can be reformulated as a \gls{LP} and trades off higher profits for a decrease in the downside risk. Robust optimization approaches were also investigated by \citet{bertsimas2006robust} and \citet{see2010robust} for a multi-period inventory problem. Finally, many authors applied data-driven distributionally robust approaches for dealing with uncertainty in the context of multi-item newsvendor problems~(see, e.g., \citealt{ben2013robust,hanasusanto2015distributionally,wang2016likelihood}, and \citealt{bertsimas2018robust}).

Despite numerous extensions to the newsvendor problem, most data-driven approaches consider only demand data but no feature variables to be available. However, ignoring the presence of features can lead to inconsistent decisions as shown in~\citet{ban2019big}. In the following, we review papers that also consider the presence of features in the context of data-driven optimization.


\vspace*{0.5cm}

\noindent\textbf{Data-driven optimization.} Beyond the newsvendor problem, some recent works have studied the integration of estimation and optimization. In particular, \citet{bertsimas2020predictive} proposed a framework for feature-based stochastic optimization problems based on a weighted \gls{SAA} approach, in which the weights are generated by machine learning methods, such as $k$-nearest neighbors regression, local linear regression, classification and regression trees, or random forests. \citet{elmachtoub2021smart} focused on problems with a linear objective and used features to learn a prediction model for the stochastic cost vector. They proposed a modified loss function that directly leverages the structure of the optimization problem instead of minimizing a standard prediction error, such as the least squares loss. Despite this modification, their approach still handles prediction and optimization as separate tasks and does not integrate them into a one-step process. \citet{mandi2020smart} further adapted the approach from \citet{elmachtoub2021smart} to solve some hard combinatorial problems, e.g., by proposing tailored warm-starting techniques.


In the context of the feature-based newsvendor, \citet{beutel2012safety} integrated estimation and optimization by learning a decision function that predicts ordering decisions directly from features, opposed to first estimating the demand and then optimizing the inventory level. The proposed model formulation is an \gls{LP} that solves an \gls{ERM} problem over a training data set. \citet{oroojlooyjadid2020applying} and \citet{zhang2017assessing} applied neural networks to the newsvendor problem, proposing specific loss functions that consider the impact of inventory shortage and holding costs. \citet{huber2019data} provided an empirical evaluation of different data-driven approaches for the feature-based newsvendor and compared their performance against \textit{model-based} approaches, which model the uncertainty through a demand distribution assumption. Their experiments on real-world data showed that data-driven approaches outperform their model-based counterparts in most cases.

Regarding feature selection, \citet{ban2019big} extended the model of~\citet{beutel2012safety} by including a regularization term to the objective function, which penalizes the complexity of the solution, thereby favoring the selection of fewer features. However, feature selection is not the main focus of~\citet{ban2019big}, and an open challenge remains regarding the specification of the regularization parameter, for which heuristics are often employed. In this work, we avoid regularization by formalizing the task of feature selection as a bilevel optimization problem for which we provide a tractable single-level reformulation.


\vspace*{0.5cm}

\noindent\textbf{Bilevel optimization in machine learning.} Bilevel optimization has been applied in the field of machine learning for hyperparameter optimization~\citep{bennett2006model, bennett2008bilevel,franceschi2018bilevel,mackay2018selftuning} and feature selection~\citep{agor2019feature}. In particular, \citet{bennett2006model,bennett2008bilevel} proposed a bilevel program for optimizing the hyperparameters of a \glsxtrlong{SVR} model. They reformulated the model into a single-level \glsxtrlong{NLP} and employed off-the-shelf solvers based on \glsxtrlong{SQP}~\citep{fletcher2002nonlinear}. \citet{franceschi2018bilevel} also proposed a bilevel programming approach for hyperparameter optimization, highlighting connections to meta-learning, and solved it with a gradient-based method.

Only~\citet{agor2019feature} addressed feature selection as a bilevel optimization problem in the context of classification models, e.g., Lasso-based logistic regression and support vector machines. However, their bilevel formulations do not apply to our problem setting, since the feature-based newsvendor combines aspects from supervised learning, i.e., regression, and data-driven optimization. Moreover, the solution method of \citet{agor2019feature} consists of a tailored genetic algorithm, which does not provide solution-quality guarantees. In contrast, our methodology is based on \gls{MILP} and allows to optimally solve the proposed bilevel programming formulations.






\subsection{Contribution}

We close the research gaps outlined above by proposing a novel bilevel optimization model that directly incorporates feature selection into solving the data-driven newsvendor problem. 
Specifically, our contribution is fourfold.
First, we introduce a bilevel program designed for feature selection, which we denote the \gls{BFS} model.
In contrast to regularization-based methods, which penalize the norm of the solution vector, \gls{BFS} captures the more intuitive notion of selecting a subset of features that minimizes an estimate of the out-of-sample cost, measured on a held-out validation set.
We reformulate the bilevel program into a single-level optimization problem, which we solve to optimality with off-the-shelf optimization solvers. 
Second, we extend the \gls{BFS} model to accommodate cross-validation strategies, which further improves its solution quality. 
Third, to illustrate the drawback of regularization-based methods for feature selection, we present a bilevel program, which we refer to as \gls{BHO}, that searches for the optimal hyperparameter for the regularized \gls{ERM} model (cf. \citealt{ban2019big}). \gls{BHO} formally describes the optimization model that established hyperparameter optimization methods implicitly solve by means of heuristics, such as grid search, random search, or Bayesian optimization.
%
%
%
Fourth, we conduct extensive numerical experiments, using synthetic instances with correlated features. We compare the proposed \gls{BFS} models against regularization-based methods in terms of out-of-sample performance and ground-truth feature recovery. We further compare the methods' behavior under demand misspecification, assuming a nonlinear demand model. Our results show that the proposed \gls{BFS} approach consistently achieves higher accuracy in feature recovery. In most cases, we also observe an improvement in out-of-sample cost performance, i.e., a decrease in test cost.


\subsection{Organization}

The remainder of this paper is structured as follows. In Section \ref{section-background}, we review the model formulations for the classical newsvendor and the feature-based newsvendor problem. Section \ref{section-methodology} presents the \gls{BHO} and the \gls{BFS} models, and consecutively extends the \gls{BFS} to cross-validation. Section \ref{section-experimental-design} describes our experimental design, and Section~\ref{section-experiments} presents the results comparing the proposed method against state-of-the-art techniques based on regularization. Section \ref{section-conclusion} concludes this paper and gives an outlook on future research.


\section{Fundamentals}\label{section-background}



In the classical newsvendor problem, a risk-neutral decision-maker sets the order quantity of a product before observing its uncertain demand. Here, the objective is to minimize the expected cost:
\begin{equation}
\min_{q \geq 0} \, \mathbb{E}\left[C(q; d)\right],
\end{equation}
where $q$ is the order quantity, $d \sim \mathcal{D}$ is the random variable representing the uncertain demand,
\begin{equation}
C(q; d) := b(d - q)^+ + h(q - d)^+ 
\end{equation}
is the cost of ordering~$q$ units and observing demand~$d$, based on the per unit shortage cost~$b$ for lost profits and unit holding cost $h$, corresponding to the procurement cost of unsold products discounted by their unit salvage value.
If the demand distribution is known, then the optimal decision~$q^*$ is given at the $b/(b+h)$ quantile of its cumulative distribution function.



In practice, the demand distribution is often not known. We consider the feature-based newsvendor problem, in which the decision-maker has access to historical demand data and contextual information given by a set of feature variables ${\mathbf{x} \in \mathbb{R}^{m+1}}$ (cf.~\citealt{beutel2012safety}). Here, the uncertain demand~$d$ and feature variables~$\mathbf{x}$ follow an (unknown) joint probability distribution $(\mathbf{x}, d) \sim \mathcal{X} \times \mathcal{D} = \mathcal{Z}$. The decision-maker's objective is to minimize the expected cost conditioned on the observed features:
\begin{equation}
\label{newsvendor-feat-pb}
\min_{q \geq 0} \, \mathbb{E}\left[C(q(\mathbf{x}); d(\mathbf{x}))| \mathbf{x}\right].
\end{equation}


One approach to solve the feature-based newsvendor is to separate the estimation and optimization problems, i.e., one first estimates the conditional demand distribution from historical data and then optimizes the order quantity based on new feature observations. One drawback of this approach is that the first step's estimation problem does not account for the asymmetry in the newsvendor cost function, related to under- and over-predicting demand. To address this issue, \citet{beutel2012safety} proposed to integrate estimation and optimization into a one-step process, by introducing a linear decision function that maps feature observations directly to ordering decisions. To learn the optimal coefficients of the decision function, one minimizes the empirical cost over a data set with demand and feature observations.


Let $\left\{(\mathbf{x}_i, d_i)\right\}_{i\in S}$ be a data set indexed by $S=\{1,\dots,n\}$, where $\mathbf{x}_i$ is an \mbox{$(m+1)$-dimensional} feature vector and $d_i$ is a scalar demand observation. Let $J = \{0, \dots, m\}$ denote the set of feature indices. 
We assume that $x^0_i = 1$ represents the feature-independent intercept term, for all $i \in S$.
In this setting, \citet{beutel2012safety} consider a linear decision function of the form:
\begin{equation}
\label{lindec-rule}
q(\mathbf{x}) = \beta^0 + \sum_{j=1}^{m} \beta^j x^j = \boldsymbol{\beta}^{\top} \mathbf{x},
\end{equation}
where $\boldsymbol{\beta} \in \mathbb{R}^{m+1}$ is the parameter vector, whose values are learned by minimizing the empirical cost on data set $S$. Upon observing new feature values, the decision-maker can then directly decide upon the order quantity instead of first estimating the uncertain demand.

Since the learned decision function may overfit to the in-sample data set $S$, it is common practice in machine learning to evaluate the out-of-sample generalization on a separate test data set~$S_{\textit{test}}$. 
To avoid overfitting and improve the out-of-sample generalization, \citet{ban2019big} proposed an extension of \citet{beutel2012safety} by integrating a regularization term into the loss function. Accordingly, the objective comprises a trade-off between minimizing the empirical in-sample cost and the regularization term, with a constant hyperparameter balancing these two terms:
\begin{flalign}
\qquad\text{(ERM-$\ell_p$)} 
\qquad\qquad && \min_{\boldsymbol{\beta}} \quad & \dfrac{1}{|S|} \sum_{i \in S} C( q_i \, ; d_i) + \lambda || \boldsymbol{\beta} ||_p \\
&& \text{s.t.} \quad & q_i = \boldsymbol{\beta}^{\top} \mathbf{x}_i & \forall i \in S, \qquad\qquad\qquad
\end{flalign}
where $\lambda \geq 0$ is the regularization hyperparameter and $|| \boldsymbol{\beta} ||_p$ is the $\ell_p$-norm of the vector $\boldsymbol{\beta}$. 
Depending on the choice of $p$ in the regularization, the resulting model may be a \gls{MILP}, an \gls{LP}, or a \glsxtrlong{SOCP}, for $\ell_0$, $\ell_1$, and $\ell_2$-norm regularization, respectively. Effectively, regularization enables feature selection by penalizing the complexity of the solution, thereby favoring sparse solution vectors.

\section{Methodology}\label{section-methodology}


We start this section presenting the \gls{BHO} model, which incorporates hyperparameter fitting in \text{(ERM-$\ell_p$)}. Then, we introduce the \gls{BFS} model as an alternative bilevel program that avoids regularization. 



\subsection{Bilevel Hyperparameter Optimization (BHO)}
\label{subsection-BHO}



In Section 2, we assumed the hyperparameter $\lambda$ as introduced in~\text{(ERM-$\ell_p$)} to be given. However, identifying $\lambda$ constitutes a challenge in itself as a respective misspecification can significantly reduce cost performance. To parametrize $\lambda$ correctly, one may utilize existing techniques for hyperparameter optimization, which partition the original data set~$S$ into a training set~$T$ and a validation set~$V$. On the training set, one learns the model parameters for a fixed hyperparameter value. Using the validation set, one can then assess the cost of the trained model for a variety of hyperparameter values, to finally choose the value~$\lambda^*$ that leads to a minimum cost on the validation set. Next, we present the \gls{BHO} formulation, which models the search for the optimal hyperparameter~$\lambda^*$ as a bilevel optimization problem.

We introduce variables~$u_i$ to model the inventory shortage and variables~$o_i$ to model the surplus inventory at the end of period $i \in T \cup V$. In the following bilevel programming formulation, the upper-level (UL) problem searches for an optimal regularization value $\lambda^* \geq 0$ that minimizes cost on the validation set $V$. In turn, the lower-level (LL) problem solves the feature-based newsvendor, as stated in \text{(ERM-$\ell_p$)}, on the training set $T$: 
\begin{flalign}
\qquad\text{(BHO-$\ell_p$ UL)} && C^*_{\text{BHO}} \,\, = \,\, \min_{\lambda \geq 0} \quad & \dfrac{1}{|V|} \sum_{i \in V} \left( b u_i + h o_i \right) \label{bho-start} \\
&&\text{s.t.} \quad & u_i \geq d_i - \boldsymbol{\beta}^{\top} \mathbf{x}_i &&\forall i \in V \label{bho-ul-u-constr}\\
&&& o_i \geq \boldsymbol{\beta}^{\top} \mathbf{x}_i - d_i &&\forall i \in V \label{bho-ul-o-constr}\\
&&& u_i \geq 0, \, o_i \geq 0 &&\forall i \in V \qquad\qquad \label{bho-ul-vars-uo} \\
&&& \boldsymbol{\beta} \in \Omega_p (\lambda), \label{bho-ul-beta-constr} 
\end{flalign}
where $\Omega_p (\lambda)$ is the set of optimal solutions $\boldsymbol{\beta}$ to the lower-level problem, parameterized by $\lambda$:
\begin{flalign}
\qquad\text{(BHO-$\ell_p$ LL)} && \Omega_p(\lambda) \,\, \coloneqq \,\, \arg \min_{\boldsymbol{\beta}} \quad & \dfrac{1}{|T|} \sum_{i \in T} \left( b u_i + h o_i \right) + \lambda || \boldsymbol{\beta} ||^2_p \label{bho-ll-obj} \\
&& \text{s.t.} \quad & u_i \geq d_i - \boldsymbol{\beta}^{\top} \mathbf{x}_i &&\forall i \in T \label{bho-ll-u-constr}\\
&&& o_i \geq \boldsymbol{\beta}^{\top} \mathbf{x}_i - d_i &&\forall i \in T \label{bho-ll-o-constr} \\
&&& u_i \geq 0, \, o_i \geq 0 &&\forall i \in T \qquad\qquad \label{bho-ll-end}
\end{flalign}
The upper-level objective~\eqref{bho-start} minimizes the newsvendor cost on the validation set~$V$ and the lower-level objective~\eqref{bho-ll-obj} minimizes the regularized newsvendor cost on the training set~$T$.
Constraints~\eqref{bho-ul-u-constr} and~\eqref{bho-ll-u-constr} define the inventory shortage for period $i \in V$ and $i \in T$, respectively, given the decision function parametrized by $\boldsymbol{\beta}$. Constraints~\eqref{bho-ul-o-constr} and~\eqref{bho-ll-o-constr} define the surplus inventory at period $i \in V$ and $i \in T$. Constraints~\eqref{bho-ul-vars-uo}, \eqref{bho-ul-beta-constr}, and \eqref{bho-ll-end} define the variable domains.


So far, we define the \gls{BHO} formulation in~\eqref{bho-start}--\eqref{bho-ll-end} for a general $\ell_p$-norm, which leads to a different model for different $p$. In the following, we illustrate some properties of \gls{BHO} under the special case of the $\ell_0$-norm regularization, which minimizes the number of non-zero elements in the $\boldsymbol{\beta}$ vector. In this case, we introduce the binary variable $z^j$ to indicate whether coefficient $\beta^j$ is non-zero. The lower-level problem can then be formulated as a \gls{MIP}:
\begin{flalign}
\qquad \text{(BHO-$\ell_0$ LL)} && \Omega_0(\lambda) \,\, \coloneqq \,\, \arg \min_{\boldsymbol{\beta}} \quad & \dfrac{1}{|T|} \sum_{i \in T} \left( b u_i + h o_i \right) + \lambda \sum_{j \in J} z^j \\
&&\text{s.t.} \quad & \text{\eqref{bho-ll-u-constr}--\eqref{bho-ll-end}} \notag\\
&&& \beta^j = 0 \text{ if } z^j = 0 && \forall j \in J \label{ind-constr}\\
&&& z^j \in \{0,1\} && \forall j \in J, \qquad\qquad
\end{flalign}
where Constraints~\eqref{bho-ll-u-constr}--\eqref{bho-ll-end} define the shortage and surplus inventory and Constraints~\eqref{ind-constr} enforce that $\beta^j = 0$ if the corresponding feature is not selected.


The \gls{BHO} formulation~\eqref{bho-start}--\eqref{bho-ll-end} generalizes many common methods for hyperparameter optimization. 
To avoid the high computational effort of solving the \gls{BHO} model to optimality, existing methods relax the assumption that $\lambda$ can take any value in $\mathbb{R}_{\geq 0}$, and consider a finite support set $\Lambda \subseteq \mathbb{R}_{\geq 0}$ instead (\citealt{bergstra2012random,pmlr-v28-bergstra13}). For example, suppose the values in $\Lambda$ are equally spaced along a grid, i.e., a line segment, then the resulting model corresponds to the well-known \textit{grid search} method. If the values in $\Lambda$ are randomly selected in a closed region, then the formulation describes the \textit{random search} method. Other approaches, e.g., based on Bayesian optimization, would perform an adaptive search, iteratively selecting a value ${\lambda}$ for the upper-level variable and then optimizing the lower-level problem. The iterative selection of new values for $\lambda$ depends on the validation performance of previously selected points.
In essence, current methods for hyperparameter optimization, such as the examples described above, are heuristics that avoid solving the \gls{BHO} model to optimality.

\subsection{Bilevel Feature Selection (BFS)}
\label{subsection-BFS}


To remedy the drawback of \gls{BHO}, we introduce a bilevel programming formulation specifically designed for feature selection. Instead of penalizing the number of selected features, we propose a more intuitive model, in which the upper-level problem selects a subset of features that minimize the empirical cost on a validation set. We then reformulate the resulting model into a single-level problem, which is computationally more tractable, and finally compare the proposed \gls{BFS} and \gls{BHO} models.

Consider our original data set $S$, which we partition into a training set $T$ and a validation set~$V$. We introduce binary variables $z^j$, for $j\in J$, to indicate whether feature $j$ is marked as relevant ($z^j=1$) or not ($z^j=0$). 
In the upper-level, \gls{BFS} selects a subset of features that minimizes the empirical cost on the validation set $V$. The lower-level problem then learns the optimal coefficients of the decision function in the training set $T$ by solving the \gls{ERM} model using only the features selected in the upper-level. We formulate the resulting upper-level problem as follows:
\begin{flalign}
\qquad \text{(BFS UL)} && C^*_{\text{BFS}} \,\, = \,\, \min_{\mathbf{z}} \quad & \dfrac{1}{|V|} \sum_{i \in V} \left( b u_i + h o_i \right) \label{bfs-start} \\
&&\text{s.t.} \quad & \text{\eqref{bho-ul-u-constr}--\eqref{bho-ul-vars-uo}} \notag\\
&&& z^j \in \{0, 1\} &&\forall j \in J \qquad\qquad \label{bfs-ul-vars-z} \\
&&& \boldsymbol{\beta} \in \Pi_0 (\mathbf{z}), \label{bfs-ul-vars-beta}
\end{flalign}
where $\Pi_0(\mathbf{z})$ is the set of optimal solutions $\boldsymbol{\beta}$ to the lower-level problem:
\begin{flalign}
\qquad \text{(BFS LL)} && \Pi_0 (\mathbf{z}) \,\, \coloneqq \,\, \arg \min_{\boldsymbol{\beta}} \quad & \dfrac{1}{|T|} \sum_{i \in T} \left( b u_i + h o_i \right) \label{bfs-ll-obj} \\
&&\text{s.t.} \quad & \text{\eqref{bho-ll-u-constr}--\eqref{bho-ll-end}} \notag\\
&&& \beta^j = 0 \text{ if } z^j = 0 && \forall j \in J \qquad\qquad \label{bfs-ll-ind-constr} 
\end{flalign}
The upper and lower-level objectives~\eqref{bfs-start} and~\eqref{bfs-ll-obj} minimize the newsvendor cost on the validation set~$V$ and training set~$T$, respectively.
Constraints~\eqref{bho-ul-u-constr}--\eqref{bho-ul-vars-uo} and \eqref{bho-ll-u-constr}--\eqref{bho-ll-end} define the shortage and surplus inventory.
Constraints~\eqref{bfs-ul-vars-z}--\eqref{bfs-ul-vars-beta} define the variable domains and Constraints~\eqref{bfs-ll-ind-constr} ensure that $\beta^j = 0$ if feature $j$ is not selected.


We reformulate Model~\eqref{bfs-start}--\eqref{bfs-ll-ind-constr} by substituting the lower-level problem by its \gls{KKT} conditions (cf. \citealt{cao2006capacitated,fontaine2014benders}). We introduce the dual variables $\mu_i$, and~$\gamma_i$ corresponding to constraints~\eqref{bho-ll-u-constr} and \eqref{bho-ll-o-constr} of the lower-level problem. The equivalent single-level (SL) optimization problem can then be expressed by using indicator constraints: 
\begin{flalign}
\qquad \text{(BFS SL)} && \min \quad & \dfrac{1}{|V|} \sum_{i \in V} \left( b u_i + h o_i \right) \label{bl-single-level-start} \\
&& \text{s.t.} \quad & \text{\eqref{bho-ul-u-constr}--\eqref{bho-ul-vars-uo}, \eqref{bho-ll-u-constr}--\eqref{bho-ll-end}, \eqref{bfs-ul-vars-z}, \eqref{bfs-ll-ind-constr}} \notag\\
&&& \dfrac{1}{|T|} \sum_{i \in T} (b u_i + h o_i) \leq \sum_{i \in T} (\gamma_i - \mu_i ) d_i \label{opt-condition-constr} \\
&&& \mu_i + \dfrac{b}{|T|} \geq 0 && \forall i \in T \label{dual-mu-constr} \\
&&& \gamma_i + \dfrac{h}{|T|} \geq 0 &&\forall i \in T \label{dual-gamma-constr} \\
&&& \sum_{i \in T} (\mu_i - \gamma_i) x_i^j = 0 \text{ if } z^j = 1 &&\forall j \in J \label{dual-beta-constr} \\
&&& \mu_i \leq 0, \, \gamma_i \leq 0 &&\forall i \in T \label{bfs-o-vars} 
\end{flalign}
%
As before, Constraints~\eqref{bho-ul-u-constr}--\eqref{bho-ul-vars-uo} and \eqref{bho-ll-u-constr}--\eqref{bho-ll-end} define the shortage and surplus inventory. Constraints~\eqref{bfs-ul-vars-z} and \eqref{bfs-ll-ind-constr} model the selection of features.
Constraint~\eqref{opt-condition-constr} represents the optimality condition of the lower-level problem, by comparing its primal objective value with the corresponding dual objective value.
Constraints~\eqref{dual-mu-constr} and~\eqref{dual-gamma-constr} are the dual constraints of the lower-level problem associated with primal variables $u_i$ and $o_i$ for $i \in T$. 
Constraints~\eqref{dual-beta-constr} are the dual constraints related to the primal variables~$\beta^j$ for $j \in J$, and Constraints~\eqref{bfs-o-vars} define the domain of the dual variables. 
The single-level reformulation has $2n+2|T|+2|J|$ variables and $2n+2|T|+2|J|+1$ constraints.



The \gls{BFS} model shares some similarities with the \gls{BHO} model. Both models have the same upper-level objective and the lower-level objectives differ only in the presence of the regularization term. We provide an overview of the main properties of both models in Table~\ref{table:BHO_vs_BFS}. The main advantage of the \gls{BFS} model regarding tractability is due to the existence of binary variables being limited to the upper-level problem. Consequently, we can reformulate the \gls{BFS} model into a \gls{MILP} and leverage the power of today's optimization software to find optimal solutions.

\begin{table}[t]
\centering
\hspace*{0.25cm}
{\linespread{1.4}\selectfont
\begin{tabular}{ >{\centering}m{2.1cm} >{\centering}m{2.1cm}| >{\centering}m{5.1cm} |>{\centering\arraybackslash}m{5.1cm}}
 \hline
 \multicolumn{2}{c|}{Formulation} & \gls{BHO} ($\ell_0$-norm reg.) & \gls{BFS} \\ 
 \hline
 \hline
 \multirow{2}{*}{Upper-level} & Objective & minimize validation cost & minimize validation cost \\ 
 \cline{2-4} 
 & Variables & $\lambda \in \mathbb{R}_{\geq 0}$ & $\mathbf{z} \in \{0,1\}^{m+1}$ \\
 \hline
 \multirow{2}{*}{Lower-level} & Objective & minimize training cost + regularization & minimize training cost \\ 
 \cline{2-4}
 & Variables & $\boldsymbol{\beta} \in \mathbb{R}^{m+1}, \, \mathbf{z} \in \{0,1\}^{m+1}$ & $\boldsymbol{\beta} \in \mathbb{R}^{m+1}$ \\
 \hline
\end{tabular}
}
\caption{Comparison between \gls{BFS} and \gls{BHO} with $\ell_0$-norm regularization}
\label{table:BHO_vs_BFS}
\end{table}

Moreover, the following results show that the optimal cost of the \gls{BFS} model is a lower bound to the optimal cost of the \gls{BHO} model when adopting $\ell_0$-norm regularization.

\begin{lemma}\label{lemma-1}
	Given a fixed selection of features $\mathbf{z}$ for both \gls{BHO} and \gls{BFS}, i.e., $\mathbf{z}_{\text{BHO}} = \mathbf{z}_{\text{BFS}} = \mathbf{z}'$ (assuming that $\mathbf{z}'$ is feasible for both problems), solving the remaining problems for the rest of the decision variables yields optimal solutions $\boldsymbol{\beta}_{\text{BHO}}\big|_{\mathbf{z}=\mathbf{z}'} = \boldsymbol{\beta}_{\text{BFS}}\big|_{\mathbf{z} = \mathbf{z}'}$
	with costs $C_{\text{BHO}} \big|_{\mathbf{z} = \mathbf{z}'} = C_{\text{BFS}}\big|_{\mathbf{z} = \mathbf{z}'}$.
\end{lemma}
\begin{proof}
By fixing $\mathbf{z}_{\text{BHO}} = \mathbf{z}'$, the regularization term in the lower-level objective becomes constant and $\lambda$ does not influence the optimal solution. Therefore, we can ignore regularization and the lower-level problem of the \gls{BHO} becomes equal to the lower-level problem of the \gls{BFS}, leading to \mbox{$\boldsymbol{\beta}_{\text{BHO}}\big|_{\mathbf{z}=\mathbf{z}'} = \boldsymbol{\beta}_{\text{BFS}}\big|_{\mathbf{z} = \mathbf{z}'}$} as the optimal solution. Since the upper-level objectives are equal in both models, the optimal costs will be equal: $C_{\text{BHO}} \big|_{\mathbf{z} = \mathbf{z}'} = C_{\text{BFS}}\big|_{\mathbf{z} = \mathbf{z}'}$. \end{proof}

\begin{proposition}\label{prop-1}
	The optimal cost of the \gls{BFS} model is a lower bound for the optimal cost of the \gls{BHO} model with $\ell_0$-norm regularization: $C_{\text{BHO}}^* \geq C_{\text{BFS}}^*$.
\end{proposition}
\begin{proof}
	Let $\mathbf{z}_{\text{BHO}} = \mathbf{z}^*$ be the optimal selection of features according to \gls{BHO} with cost $C_{\text{BHO}}^*$. Suppose that $\mathbf{z}_{\text{BFS}} = \mathbf{z}'$ is a solution to \gls{BFS}, such that $C_{\text{BHO}}^* < C_{\text{BFS}}\big|_{\mathbf{z}=\mathbf{z}'}$. We can always improve the cost of \gls{BFS} by setting $\mathbf{z}_{\text{BFS}} = \mathbf{z}^*$ in the upper-level problem. Because of Lemma~\ref{lemma-1}, this will result in a new solution with cost $C_{\text{BFS}}\big|_{\mathbf{z}=\mathbf{z}^*} = C_{\text{BHO}}^* \geq C^*_{\text{BFS}}$. 
\end{proof}

\subsection{Bilevel Feature Selection with Cross-Validation (BFS-CV)}
\label{subsection-BFS-CV}


Cross-validation strategies often improve the generalization ability of machine learning models and prevent overfitting by using data re-sampling methods. Accordingly, we extend the \gls{BFS} model to cross-validation instead of simple hold-out validation. We consider $K$ training-validation splits of the data and search for the set of features that minimize the average cost over all $K$ validation sets. For each $k \in [K] = \{1, \dots, K\}$, we consider a subset of observations $S_k \subseteq S$ sampled from the original data set $S$. Analogously to \gls{BFS}, we partition the set $S_k$ into a training set $T_k$ and a validation set~$V_k$. 
We introduce variables $u_{ik}$ and $o_{ik}$ to model the inventory shortage and surplus, respectively, at the end of period $i\in T_k\cup V_k$ for each training-validation split $k\in[K]$.
We then learn the model parameters $\boldsymbol{\beta}_k \in \mathbb{R}^{m+1}$ using the corresponding training set $T_k$, and select features by minimizing the average validation cost over all validation sets $V_k$ for $k \in [K]$. The resulting problem is a bilevel program with $K$ lower-level problems:
\begin{flalign}
\qquad \text{(BFS-CV UL)} && \min \quad & \dfrac{1}{K} \sum_{k = 1}^{K} \dfrac{1}{|V_k|} \sum_{i \in V_k} \left( b u_{ik} + h o_{ik} \right) \label{bfs-cv-start}\\
&& \text{s.t.} \quad & u_{ik} \geq d_{i} - \boldsymbol{\beta}_{k}^{\top} \mathbf{x}_i &&\forall \, k \in [K], \forall i \in V_k \label{bfs-cv-ul-u-constr}\\
&&& o_{ik} \geq \boldsymbol{\beta}_k^{\top} \mathbf{x}_i - d_{i} &&\forall \, k \in [K], \forall i \in V_k \label{bfs-cv-ul-o-constr}\\
&&& u_{ik} \geq 0, \, o_{ik} \geq 0 &&\forall \, k \in [K], \forall i \in V_k \qquad \label{bfs-cv-vars-uo}\\
&&& z^j \in \{0, 1\} &&\forall j \in J \label{bfs-cv-vars-z}\\
&&& \boldsymbol{\beta}_k \in \Pi_k (\mathbf{z}) && \forall \, k \in [K], \label{bfs-cv-vars-beta}
\end{flalign}
where $\Pi_k(\mathbf{z})$ is the set of optimal solutions corresponding to the $k^{\text{th}}$ lower-level problem:
\begin{flalign}
\qquad \text{(BFS-CV LL)} && \Pi_k(\mathbf{z}) \,\, \coloneqq \,\, \arg \min_{\boldsymbol{\beta}_k} \quad & \dfrac{1}{|T_k|} \sum_{i \in T_k} \left( b u_{ik} + h o_{ik} \right) \\
&&\text{s.t.} \quad & u_{ik} \geq d_i - \boldsymbol{\beta}_k^{\top} \mathbf{x}_i &&\forall i \in T_k \label{bfs-cv-ll-u-constr}\\
&&& o_{ik} \geq \boldsymbol{\beta}_k^{\top} \mathbf{x}_i - d_i &&\forall i \in T_k \label{bfs-cv-ll-o-constr}\\
&&& u_{ik} \geq 0, \, o_{ik} \geq 0 &&\forall i \in T_k \qquad\qquad \label{bfs-cv-ll-vars-uo} \\
&&& \beta_k^j = 0 \text{ if } z^j = 0 &&\forall j \in J \label{bfs-cv-ll-ind-constr}
\end{flalign}
Constraints~\eqref{bfs-cv-ul-u-constr}--\eqref{bfs-cv-vars-uo} and~\eqref{bfs-cv-ll-u-constr}--\eqref{bfs-cv-ll-vars-uo} define the shortage and surplus inventory for period $i\in V_k$ and $i\in T_k$, respectively for each split $k\in[K]$. 
Constraints~\eqref{bfs-cv-vars-z} and \eqref{bfs-cv-vars-beta} define the variable domains and Constraints~\eqref{bfs-cv-ll-ind-constr} ensure that $\beta^j_k=0$ if feature $j$ is not selected for the training-validation split~$k$.

The above model can accommodate different cross-validation strategies, such as \mbox{$K$-fold}, random permutations (Shuffle \& Split), or Leave-P-Out cross-validation (see, e.g., \citealt{arlot2010survey,hastie2009elements}). Each particular choice of cross-validation strategy corresponds to a different approach for constructing the subsets $S_k$ and partitioning the data into $T_k$ and $V_k$. Moreover, the special case with $K=1$ corresponds to the previously introduced \gls{BFS} model.


Analogously to \gls{BFS}, \gls{BFS-CV} can be reformulated into a single-level \gls{MILP}:
\begin{flalign}
\quad \text{(BFS-CV SL)} && \min_{\mathbf{z}} \quad & \dfrac{1}{K} \sum_{k = 1}^{K} \dfrac{1}{|V_k|} \sum_{i \in V_k} \left( b u_{ik} + h o_{ik} \right) \label{bfs-cv-single-level-start} \\
&& \text{s.t.} \quad & \text{\eqref{bfs-cv-ul-u-constr}--\eqref{bfs-cv-vars-z}, \eqref{bfs-cv-ll-u-constr}--\eqref{bfs-cv-ll-ind-constr}}\notag\\
&&& \dfrac{1}{|T_k|} \sum_{i \in T_k} (b u_{ik} + h o_{ik}) \leq \sum_{i \in T_k} (\gamma_{ik} - \mu_{ik}) d_i && \forall \, k \in [K] \label{bfs-cv-opt-cond-constr}\\
&&& \mu_{ik} + \dfrac{b}{|T_k|} \geq 0 &&\forall k \in [K], \, \forall i \in T_k \label{bfs-cv-dual-mu-constr}\\
&&& \gamma_{ik} + \dfrac{h}{|T_k|} \geq 0 &&\forall k \in [K], \, \forall i \in T_k \label{bfs-cv-dual-gamma-constr}\\
&&& \sum_{i \in T_k} (\mu_{ik} - \gamma_{ik}) x_i^j = 0 \text{ if } z^j = 1 &&\forall k \in [K], \, \forall j \in J \label{bfs-cv-dual-beta-constr}\\
&&& \mu_{ik} \leq 0, \, \gamma_{ik} \leq 0 &&\forall k \in [K], \, \forall i \in T_k \label{bfs-cv-dual-vars}
\end{flalign}
Constraints~\eqref{bfs-cv-opt-cond-constr} represents the optimality condition of the lower-level problem, for each training-validation split $k\in[K]$. Constraints~\eqref{bfs-cv-dual-mu-constr} and~\eqref{bfs-cv-dual-gamma-constr} are the dual constraints of the lower-level problem associated with primal variables~$u_{ik}$ and $o_{ik}$ for $i\in T_k$ for $k\in[K]$. Constraints~\eqref{bfs-cv-dual-beta-constr} are the dual constraints related to the primal variables $\boldsymbol{\beta}^k$ for $k \in [K]$, and Constraints~\eqref{bfs-cv-dual-vars} define the domain of the dual variables. The single-level reformulation has \mbox{$K (4|T|+2|V|+|J|) + |J|$} variables and $K(4|T|+2|V|+2|J|+1)$ constraints.

\section{Experimental Design}\label{section-experimental-design}

To benchmark our approaches, we perform extensive computational experiments on synthetic instances.
The goals of our computational experiments are fourfold.
\begin{enumerate}[(i)]
	\item We evaluate the performance of a \gls{MILP} approach for the proposed \gls{BFS} and \gls{BFS-CV} models in terms of feature selection and generalization to out-of-sample data;
	\item We compare the performance of our approach against existing regularization-based methods;
	\item For each method, we compare the performance of hold-out validation against cross-validation;
	\item We analyze the effect of different instance parameters on each method's performance.
\end{enumerate}

We implemented all methods in \texttt{C++}, using \mbox{CPLEX 20.1} to solve the respective \gls{MILP} formulations. All experiments were performed on machines with Intel Core i7-6700 CPU at 3.40 GHz, with 16 GB of RAM, and Ubuntu 16.04.6 LTS operating system, under a time limit of 900 seconds. We provide the source code and data at [to be disclosed after peer-review].

\subsection{Instances}\label{section-instances}

We adapt the experimental setup from \citet{zhu2012semiparametric} and consider a linear demand model with additive noise:
\begin{equation}\label{linear-demand}
d_{\text{linear}}(\mathbf{x}) = 5 + \boldsymbol{\beta}^{\top} \mathbf{x} + \epsilon,
\end{equation}
%
where $\boldsymbol{\beta} = (0, 2, -2, -1, 1, 0, \dots, 0)^{\top}/ \sqrt{10}$ is an $(m+1)$-dimensional vector representing the ground-truth coefficients. The feature variables $\mathbf{x} \in \mathbb{R}^{m+1}$ are drawn from a multivariate Gaussian distribution with mean zero and covariance matrix with entries $\sigma_{ij} = 0.5^{|i-j|}$, for $(i,j) \in J^2$. 
The noise term follows a Gaussian distribution $\epsilon \sim \mathcal{N}(0, \sigma^2_{\epsilon})$.
We set negative demand values to zero.

We generate instances varying the number of samples $n$ from 40 to 2000 and the number of features~$m$ from 8 to 14. For each configuration, we generate 20 instances to account for variability in the distributions.
Additionally, we generate a separate test set with 1000 observations associated with each instance, following the same distributions.

Furthermore, we analyze the impact of demand misspecification, i.e., when the demand is not linearly related to the features. We investigate the following nonlinear demand model (cf. \citealt{zhu2012semiparametric}):
\begin{equation}\label{nonlinear-demand}
d_{\text{nonlinear}}(\mathbf{x}) = 10 + \sin \big(2(\boldsymbol{\beta}^{\top} \mathbf{x})\big) + 2 \exp \big( -16(\boldsymbol{\beta}^{\top} \mathbf{x})^2 \big) + \varphi(\boldsymbol{\beta}^{\top} \mathbf{x}) \epsilon,
\end{equation}
where $\varphi(\boldsymbol{\beta}^{\top} \mathbf{x})=1$ for a homoscedastic case and $\varphi(\boldsymbol{\beta}^{\top} \mathbf{x}) = \exp (\boldsymbol{\beta}^{\top} \mathbf{x})$ for a heteroscedastic case.

We solve the proposed \gls{BFS} and \gls{BFS-CV} models and compare the results against regularization-based methods from the literature. Since our motivation for feature selection is to provide more explainable decisions, we focus on methods based on linear decision functions, which are intrinsically more explainable.
We consider the \gls{ERM} model of \citet{ban2019big} with $\ell_0$ and $\ell_1$-norm regularization and use grid search with 50 break-points to calibrate the regularization parameter. We run each considered method once using a hold-out validation set and once with Shuffle \& Split cross-validation (CV). For hold-out validation, we use half of the samples in each instance as a training set and the other half as a validation set, following the setting in~\cite{zhu2012semiparametric}. For cross-validation, we perform $K=50$ re-sampling iterations, where we sample a subset of size $|S_k| = \min\{200,n\}$, for each $k \in [K]$.

\subsection{Performance metrics}\label{section-performance-metrics}

We assess the ability of each method to recover the ground-truth feature vector, adopting the accuracy measure:
\begin{equation}
\alpha = \dfrac{1}{m} \sum_{j=1}^{m} \mathbbm{1} (\hat{z}_j = z^*_j),
\end{equation}
where $\mathbbm{1}(\cdot)$ is the indicator function, $\hat{\mathbf{z}} = [\hat{z}_1, \dots, \hat{z}_m]$ is the estimated binary feature vector, and $\mathbf{z}^* = [z^*_1, \dots, z^*_m]$ is the ground-truth vector, defined as $z^*_j=0$ if $\beta^*_j=0$, i.e., feature $j$ is non-informative, otherwise $z^*_j=1$. 
Our definition of accuracy is analogous to the one commonly adopted for binary classification, where $z^*_j \in \{0,1\}$ represents the class assigned to feature~$j$.

Additionally, we analyze the cost performance of applying the learned decision functions to out-of-sample data. 
We therefore evaluate the out-of-sample cost on a separate test data set with 1000 observations. We report the test cost values of each method $\mathcal{M}$ in terms of its percentage deviation relative to the test cost achieved by \gls{BFS-CV} on the same configuration:
\begin{equation}\label{test-cost-deviations}
\delta(\mathcal{M}) = 100 \times \frac{ C_{\mathcal{M}}^{\textit{test}} - C^{\textit{test}}_{\text{BFS-CV}} }{C^{\textit{test}}_{\text{BFS-CV}}},
\end{equation}
where $C_{\mathcal{M}}^{\textit{test}}$ is the average newsvendor cost of method $\mathcal{M}$ calculated on the test set. Deviation values greater than zero indicate that \gls{BFS-CV} improves upon method $\mathcal{M}$ regarding test cost performance, while values below zero indicate that method $\mathcal{M}$ achieves lower test cost than \gls{BFS-CV}.

\section{Results}\label{section-experiments}


First, we present results concerning instances generated by the linear demand model, and then discuss results on instances with nonlinear demand. 

\subsection{Linear demand model}
\label{section-experiments-linear}

We analyze how instance size, number of features, noise level, shortage cost, and holding cost affect the performance of each method. Unless otherwise stated, we use a setting with $n=1000$ samples, $m=10$ features, a shortage cost of~$b=2$, and a holding cost of~$h=1$ as reference configuration. For the noise term, we consider $\sigma_{\epsilon}=1$ as reference configuration for the Gaussian distribution.
We provide results regarding computation times in~\ref{appendix-computation-times}.

\vspace*{0.5cm}


\noindent\textbf{Instance size.} Figure~\ref{fig:accuracy_feature_recovery} reports the feature recovery accuracy of the different methods for a varying number of samples~$n\in[40,2000]$, averaged over $20$ randomly generated instances. In general, \mbox{\gls{BFS-CV}} achieves the highest accuracy among all methods and faster convergence for increasing~$n$, with accuracy values above $0.9$ already for instances with $200$ samples. In contrast, existing methods often yield accuracy values below $0.9$ and fail to recover the ground-truth features accurately, even for instances with a larger number of samples. We confirm these results at $5\%$ significance level by pairwise Wilcoxon signed-rank tests, and refer to~\ref{appendix-linear} for details on the respective p-values. 


\begin{figure}[!ht]
	\centerline{\includegraphics[width=0.55\linewidth]{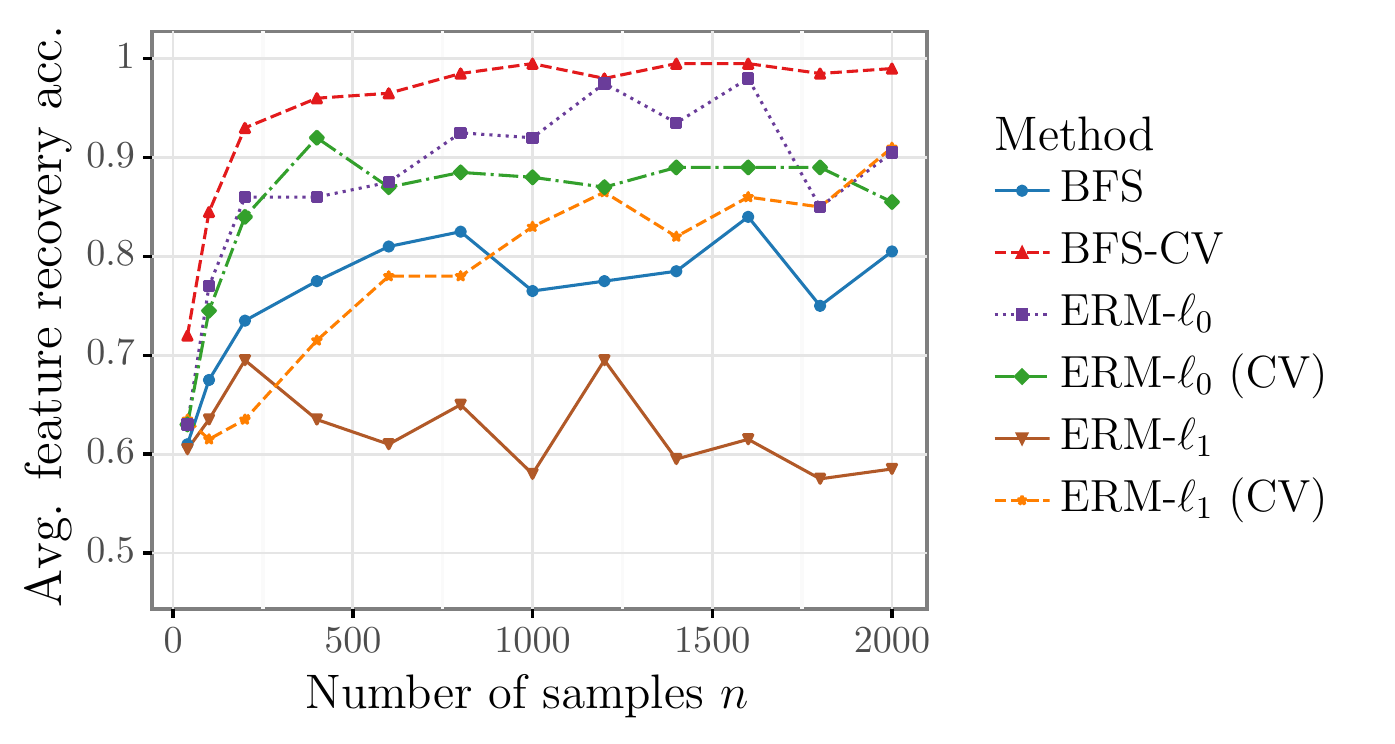}}
	\vspace{-0.2cm}
	\caption{Impact of instance size on the accuracy of feature recovery}
	\label{fig:accuracy_feature_recovery}
\end{figure}

Besides feature recovery, we evaluate the out-of-sample cost performance of each method. Figure~\ref{fig:relative_test_cost-nb_samples} shows the distribution of percentage deviations, where a positive deviation indicates that \gls{BFS-CV} is superior to the respective other method. We split the results in three different plots based on the sample size~$n$, classifying each instance as small, medium, or large. \gls{BFS-CV} outperforms the other methods in most cases, as the lower quartiles are always above or close to zero.
For smaller instances, test cost deviations can be as high as $30\%$ in the best case. 
As we increase~$n$, all methods present improving test cost performance and the variance in the test cost distribution decreases. 
Yet, \gls{BFS-CV} is still superior to the other methods in the wide majority of cases.
For large instances, all methods present mostly positive test cost deviations, with values ranging from $-1\%$ to $8\%$.



\begin{figure}[!t]
	\hspace{-0.6cm}
	\begin{tabular}{ccc}
		\includegraphics[width=0.36\linewidth]{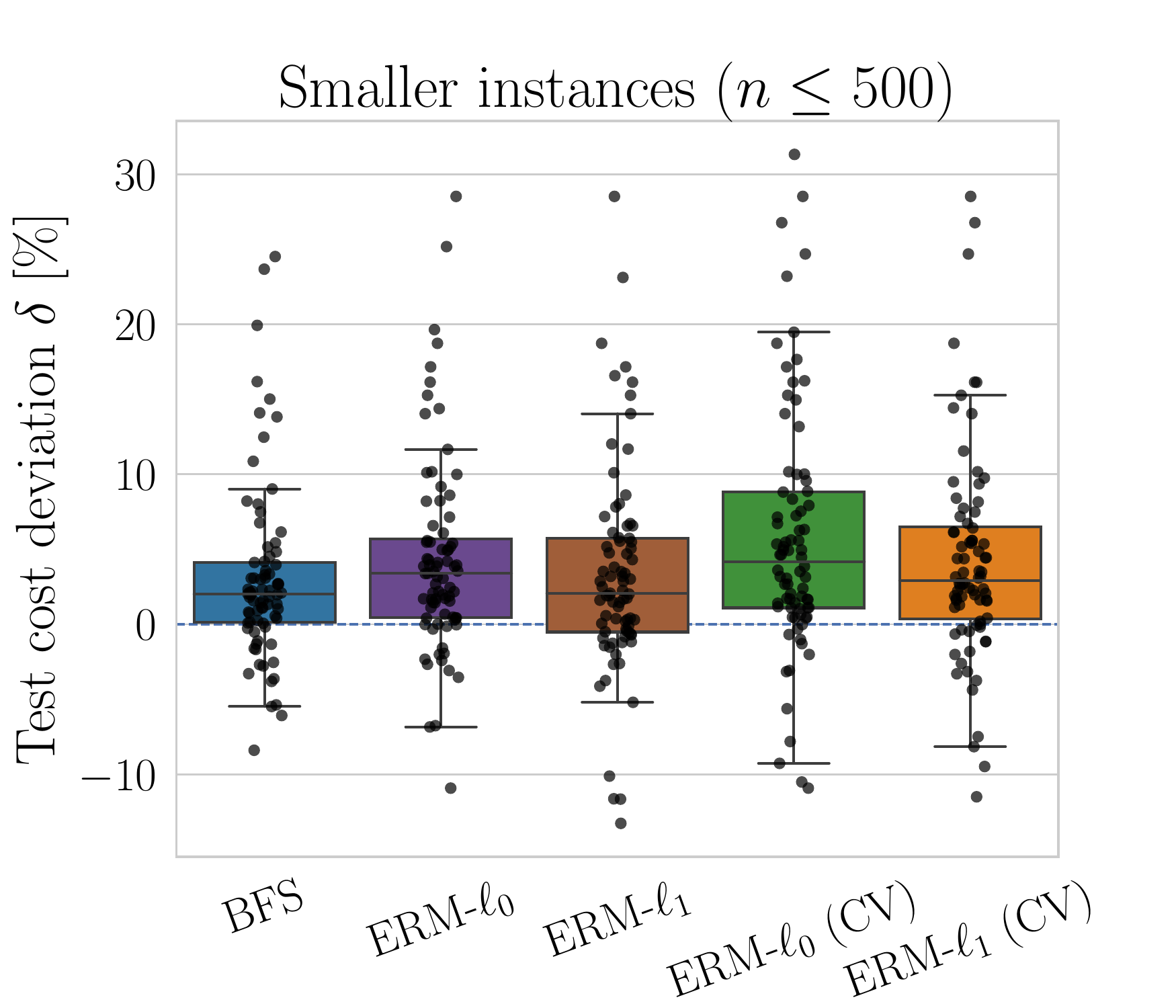} & \hspace{-0.7cm}\includegraphics[width=0.36\linewidth]{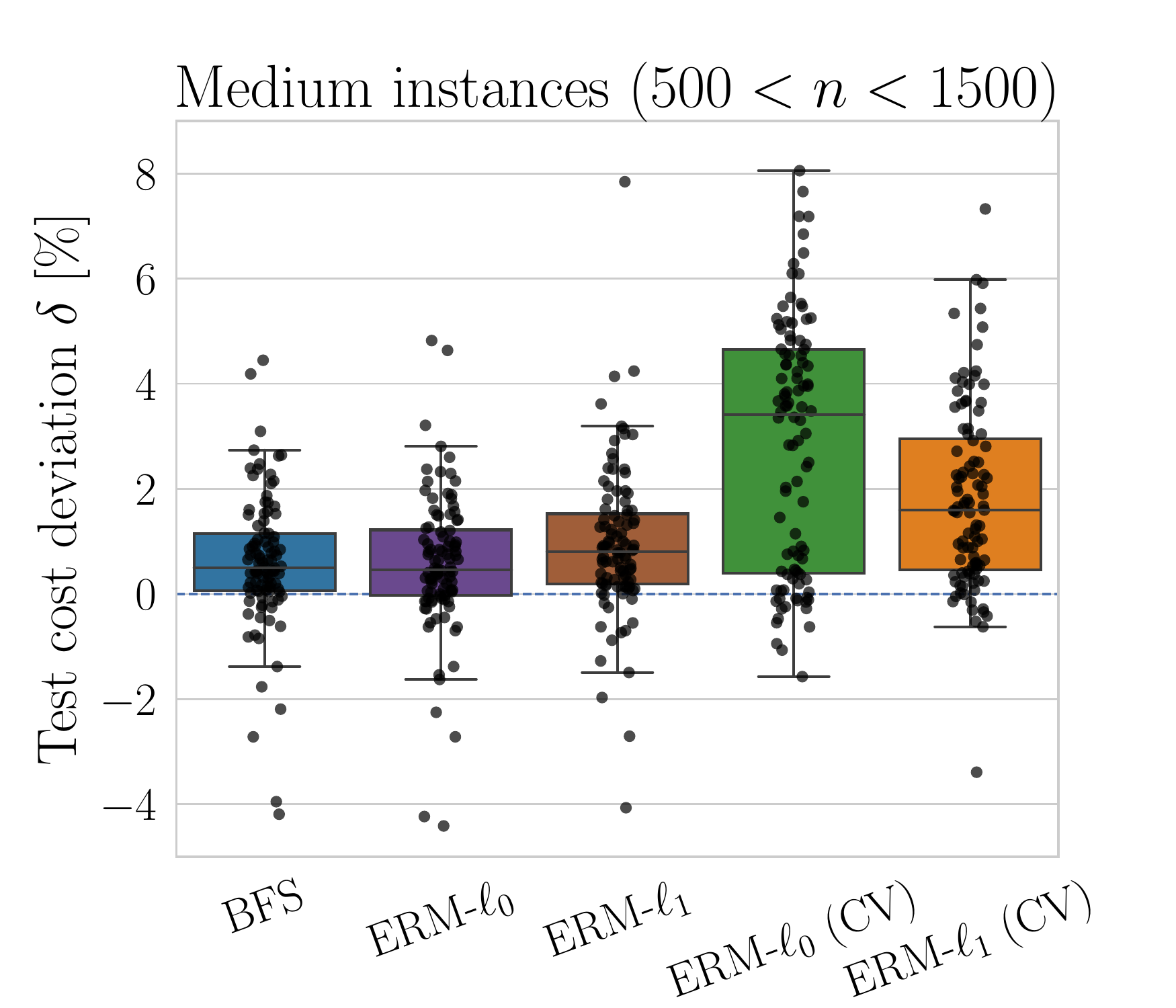} & \hspace{-0.7cm}\includegraphics[width=0.36\linewidth]{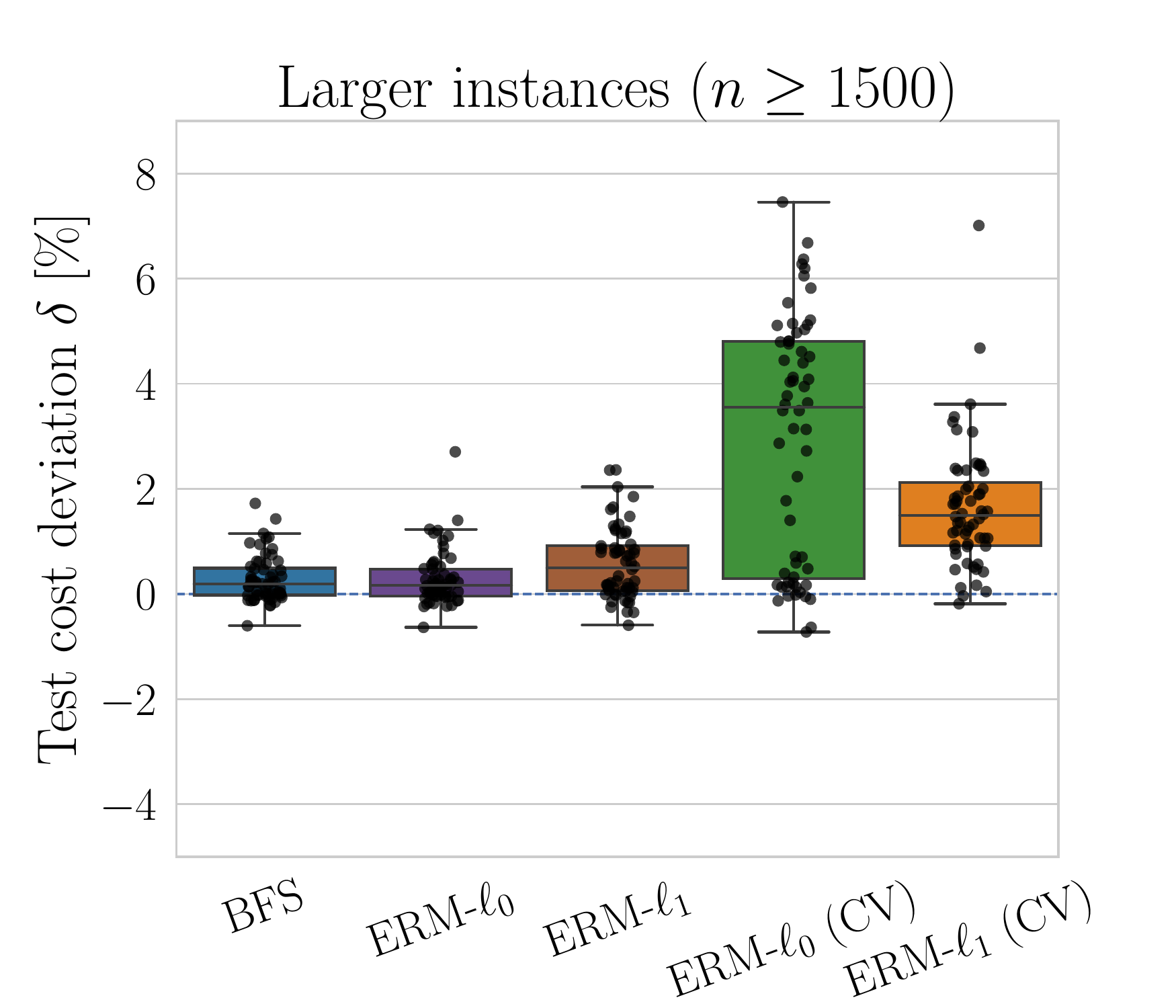} 
	\end{tabular}
	\vspace{-0.2cm}
	\caption{Impact of instance size on the test cost performance of different methods, relative to \gls{BFS-CV}}
	\label{fig:relative_test_cost-nb_samples}
\end{figure}

\noindent\textbf{Number of features.} Figure~\ref{fig:accuracy_feature_recovery-features} shows average feature recovery accuracy values, where we now fix the number of samples to $n=1000$ and vary the number of features $m \in \{8,10,12,14\}$. For all methods, the number of features does not strongly affect the accuracy performance. Notably, \gls{BFS-CV} achieves average accuracy values consistently above $0.95$, being superior to the other methods, as confirmed by pairwise Wilcoxon tests.

\begin{figure}[!ht]
	\centerline{\includegraphics[width=0.55\linewidth]{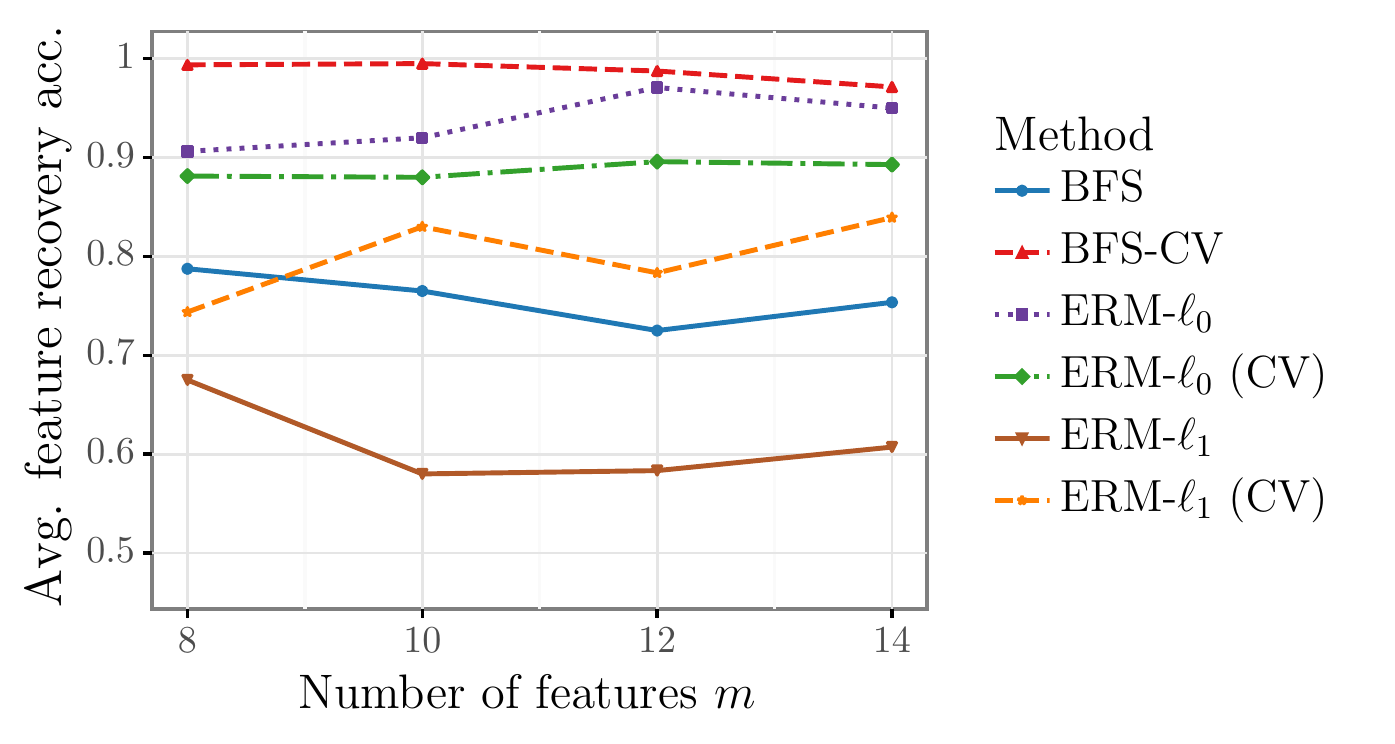}}
	\vspace{-0.2cm}
	\caption{Impact of number of features $m$ on the accuracy of feature recovery}
	\label{fig:accuracy_feature_recovery-features}
\end{figure}

Regarding test cost performance, Figure~\ref{fig:relative_test_cost-features} shows the distribution of test cost deviations. We focus on large instances ($n\geq1500$) in this analysis, which have lower variance, so that we can isolate the impact of $m$ on the test cost. For \gls{BFS}, \gls{ERM}-$\ell_0$, and \gls{ERM}-$\ell_1$, the number of features has no strong influence on the test cost deviations. In contrast, the performance of \gls{ERM}-$\ell_0$ (CV) and \gls{ERM}-$\ell_1$ (CV) shows increasing deviation values for increasing~$m$. In the majority of cases, \gls{BFS-CV} outperforms the other considered methods. 


\begin{figure}[!t]
	\centerline{\includegraphics[width=0.75\linewidth]{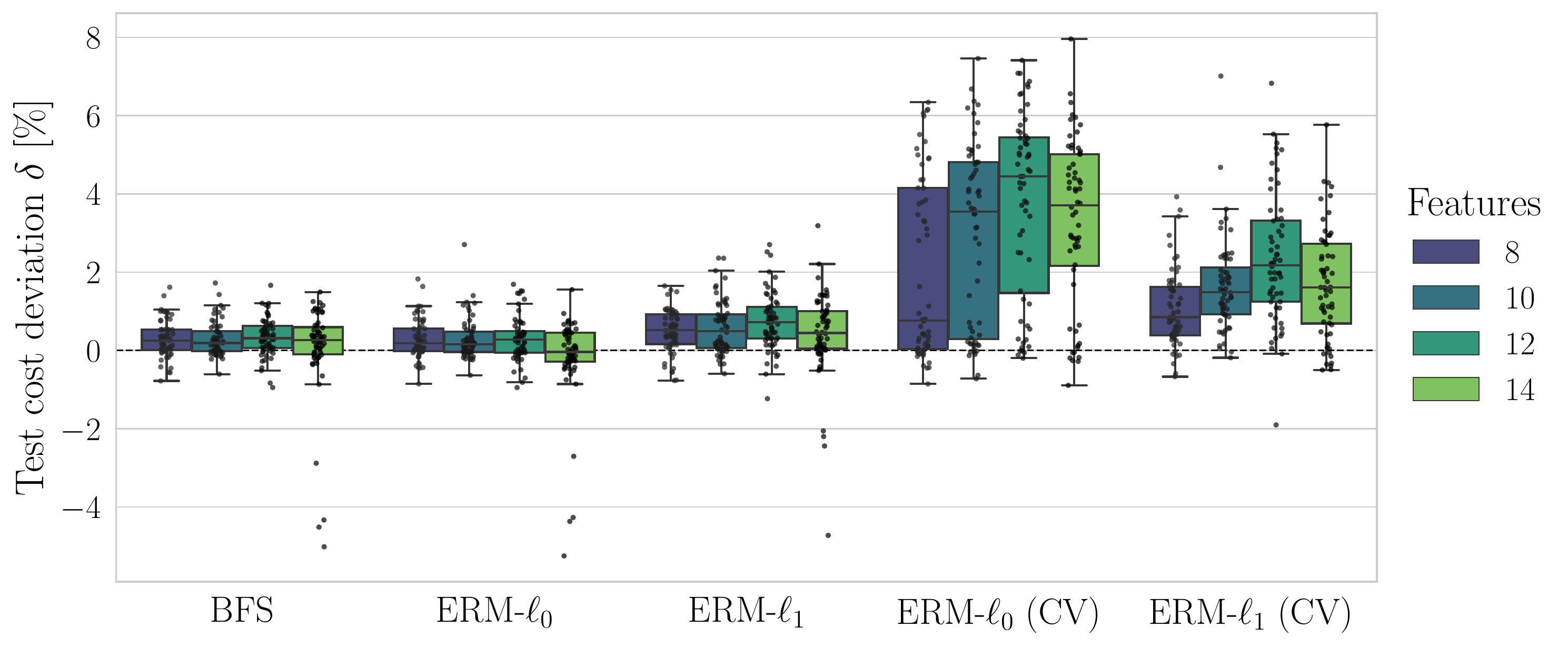}}
	\vspace{-0.2cm}
	\caption{Impact of number of features $m$ on the test cost performance of different methods, relative to \gls{BFS-CV}}
	\label{fig:relative_test_cost-features}
\end{figure}



\noindent\textbf{Noise level.} We vary the coefficient of variation $c_{\text{v}} = \sigma_{\epsilon} / \mu \in [0.2, 1]$ and report the average feature recovery accuracy for each method in Figure~\ref{fig:accuracy_feature_recovery-noise-level}. As we increase the level of noise, it becomes harder to recover the informative features and the accuracy of all methods deteriorates. For $c_{\text{v}}\in[0.2, 0.6]$, \mbox{\gls{BFS-CV}} achieves the highest accuracy among the considered methods.
Outside this range, \mbox{\gls{BFS-CV}} is outperformed by \gls{ERM}-$\ell_0$ (CV) for $c_{\text{v}}=0.1$ and by \gls{ERM}-$\ell_0$ for $c_{\text{v}} \geq 0.8$, respectively.
Still, \mbox{\gls{BFS-CV}} generally attains comparatively high accuracy, being superior to most other methods.

\begin{figure}[!ht]
	\centerline{\includegraphics[width=0.55\linewidth]{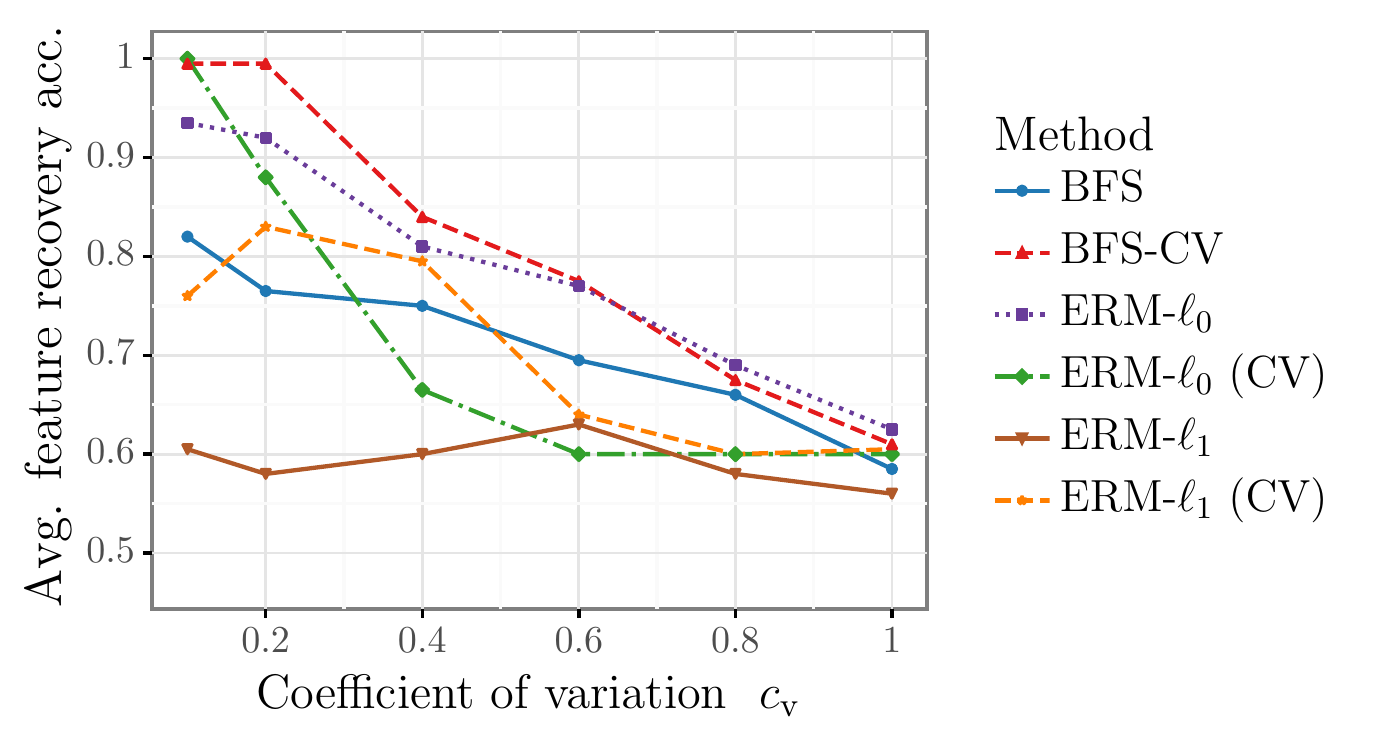}}
	\vspace{-0.2cm}
	\caption{Impact of different noise levels on the accuracy of feature recovery}
	\label{fig:accuracy_feature_recovery-noise-level}
\end{figure}

Figure~\ref{fig:relative_test_cost-noise} illustrates the impact of different noise levels on test cost performance, considering large instances ($n\geq1500$). 
Methods \gls{BFS}, \gls{ERM}-$\ell_0$, and \gls{ERM}-$\ell_1$ mostly outperform \gls{BFS-CV} for \mbox{$c_{\text{v}}\geq 0.4$}. In such cases, test cost deviations range from $-4\%$ to $4\%$, indicating that \gls{BFS-CV} achieves comparable results even when its performance is inferior to other methods.
Methods \mbox{\gls{ERM}-$\ell_0$} (CV) and \mbox{\gls{ERM}-$\ell_1$} (CV) perform comparatively worse, with mostly positive deviations and larger variance in the distribution. 

\begin{figure}[!t]
	\centerline{\includegraphics[width=0.85\linewidth]{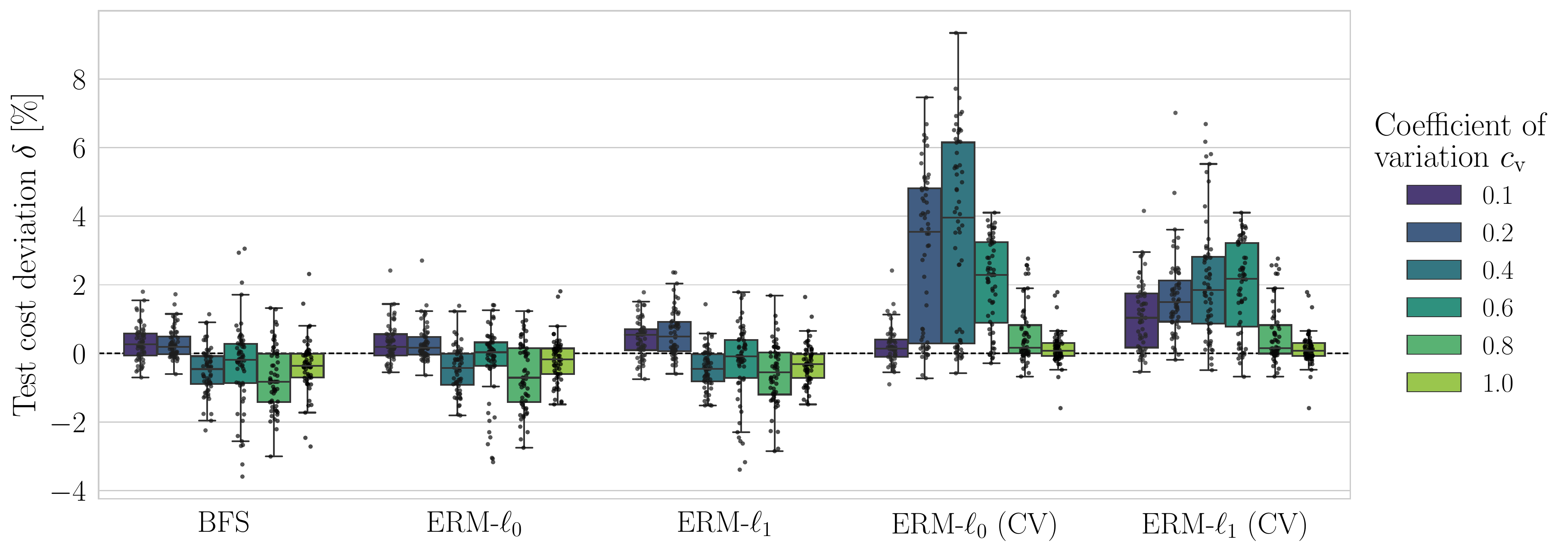}}
	\vspace{-0.2cm}
	\caption{Impact of noise level on the test cost performance of different methods, relative to \gls{BFS-CV}}
	\label{fig:relative_test_cost-noise}
\end{figure}

\vspace*{0.5cm}

\noindent\textbf{Shortage and holding costs.} Figure~\ref{fig:accuracy_feature_recovery-backorder_cost} displays results on the accuracy performance as a function of the newsvendor ratio $b/(b+h)$, by varying $(b,h)\in[1,10]^2$. In general, \gls{BFS-CV} has accuracy values consistently above $0.9$ and outperforms the other methods.

\begin{figure}[htbp]
	\centerline{\includegraphics[width=0.58\linewidth]{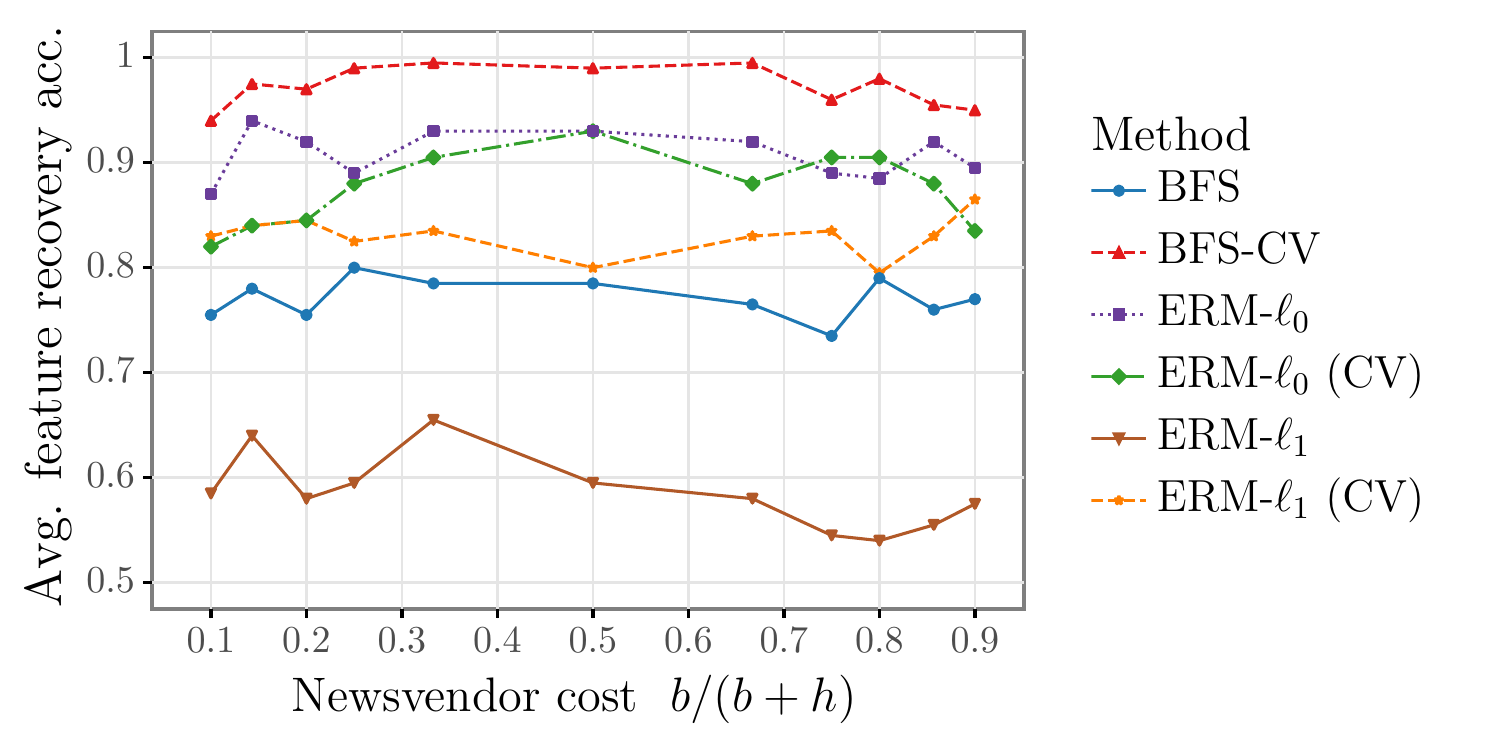}}
	\vspace{-0.2cm}
	\caption{Impact of shortage cost $b$ and holding cost $h$ on the accuracy of feature recovery}
	\label{fig:accuracy_feature_recovery-backorder_cost}
\end{figure}

Figure~\ref{fig:relative_test_cost-backorder_cost-linear_demand-gaussian_noise} shows results on test cost performance of each method as a function of $b$, for large instances ($n\geq1500$), where we fixed $h=1$, corresponding to newsvendor ratios $b/(b+h) \geq 0.5$. 
\gls{ERM}-$\ell_0$ (CV) and \gls{ERM}-$\ell_1$ (CV) generally perform worse than \gls{BFS-CV}, since the deviation values are often positive. For all considered methods, we observe that the test cost deviations range from $-6\%$ to $10\%$ and the variance increases with increasing $b$. We observed similar results for cases with $b/(b+h)<0.5$.

\begin{figure}[htbp]
	\centerline{\includegraphics[width=0.85\linewidth]{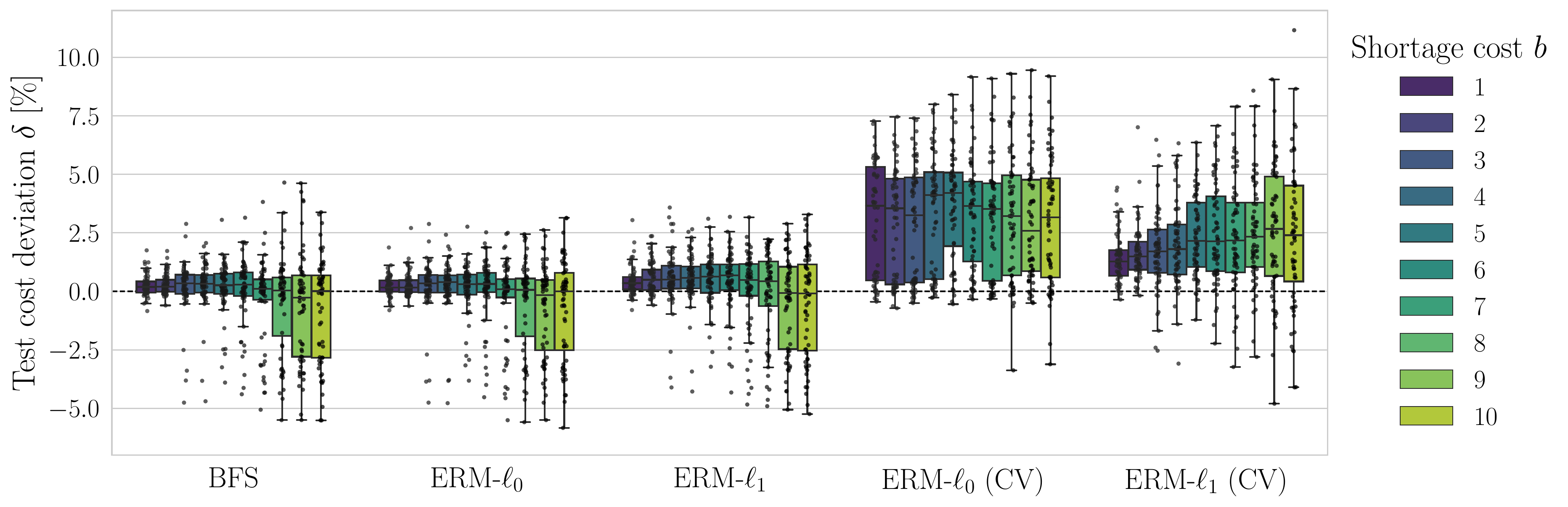}}
	\vspace{-0.2cm}
	\caption{Impact of shortage cost $b$ on the test cost performance of different methods, relative to \gls{BFS-CV}}
	\label{fig:relative_test_cost-backorder_cost-linear_demand-gaussian_noise}
\end{figure}

\subsection{Nonlinear demand model}
\label{section-experiments-nonlinear}

Demand may not be linearly related to the features. Therefore, we investigate how each method performs under nonlinear demand models, considering homoscedastic and heteroscedastic settings. 
In the following, we only present results regarding accuracy performance. Results on test cost performance did not provide new insights, since we observed similar patterns as in the case of linear instances (see~\ref{appendix-nonlinear}). 
Unless otherwise stated, we use the same reference configuration as in the previous section.

\vspace*{0.5cm}

\noindent\textbf{Instance size.} Due to the nonlinear structure of the demand models, Figure~\ref{fig:accuracy_feature_recovery-nonlinear} shows considerably lower accuracy values compared to instances with linear demand. For heteroscedastic instances, \gls{BFS-CV} outperforms existing methods, with accuracy values above $0.9$ already for $n=500$ samples.
For the homoscedastic case, all methods present inferior accuracy compared to the heteroscedastic case. In this setting, \gls{BFS-CV} and \gls{ERM}-$\ell_0$ present comparable performance, superior to the other considered methods.  
%
%
\begin{figure}[!ht]
	\centerline{\includegraphics[width=0.78\linewidth]{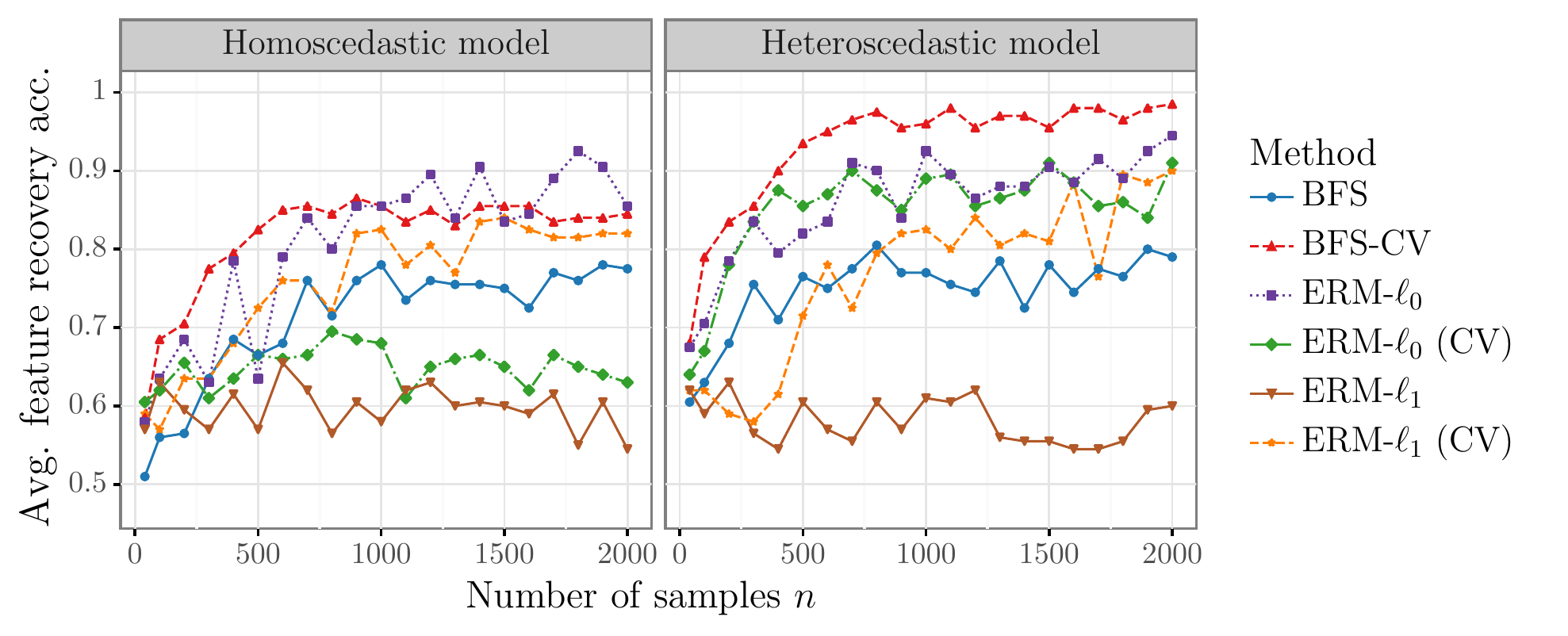}}
	\vspace{-0.2cm}
	\caption{Impact of instance size on the accuracy of feature recovery}
	\label{fig:accuracy_feature_recovery-nonlinear}
\end{figure}


\vspace*{0.5cm}

\noindent\textbf{Number of features.} Figure~\ref{fig:accuracy_feature_recovery-features-nonlinear} shows the average feature recovery accuracy for varying $m \in \{8,10,12,14\}$, for both homoscedastic and heteroscedastic demand. Similarly as for instances with linear demand, the number of features does not strongly influence the accuracy performance for nonlinear instances. 

\begin{figure}[!ht]
	\centerline{\includegraphics[width=0.78\linewidth]{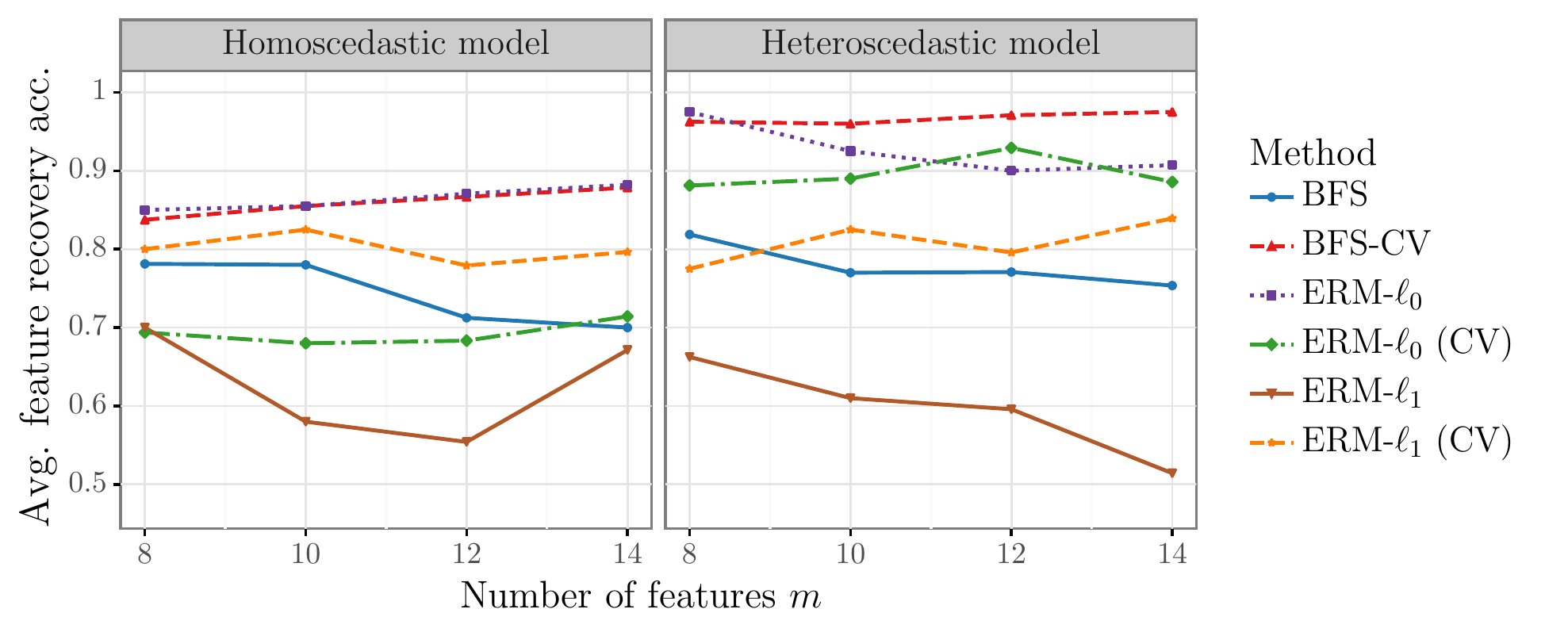}}
	\vspace{-0.2cm}
	\caption{Impact of number of features $m$ on the accuracy of feature recovery}
	\label{fig:accuracy_feature_recovery-features-nonlinear}
\end{figure}

\newpage

\noindent\textbf{Noise level.} For the heteroscedastic case, \gls{BFS-CV} outperforms existing methods (Figure~\ref{fig:accuracy_feature_recovery-noise-level-nonlinear}). In this setting, varying the standard deviation $\sigma_{\epsilon}$ does not strongly affect the accuracy performance. For homoscedastic instances, all methods present inferior accuracy compared to heteroscedastic instances. 
Here, \gls{ERM}-$\ell_0$ is superior to other methods for most values of $\sigma_{\epsilon}$, but~\gls{BFS-CV} often presents comparable accuracy.

\vspace*{0.5cm}

\noindent\textbf{Shortage and holding costs.} In Figure~\ref{fig:accuracy_feature_recovery-backorder_cost-nonlinear}, for heteroscedastic instances, \gls{BFS-CV} has superior accuracy than other methods when $b/(b+h)\geq 0.3$. We observe that all methods perform poorly when the newsvendor ratio is below 0.3. In particular, the considered methods often do not recover any features when $b/(b+h) \in [0.1,0.2]$, leading to accuracy values between 0.5 and 0.65. For the homoscedastic case, the accuracy of all methods decreases with increasing newsvendor ratios, with values often below 0.8. In this setting, \gls{BFS-CV} is superior to most methods, but is outperformed by \gls{ERM}-$\ell_0$ when $b/(b+h)\geq 0.8$ and by \gls{ERM}-$\ell_1$ (CV) when $b/(b+h)\in[0.2,0.35]$.

\begin{figure}[!t]
	\centerline{\includegraphics[width=0.78\linewidth]{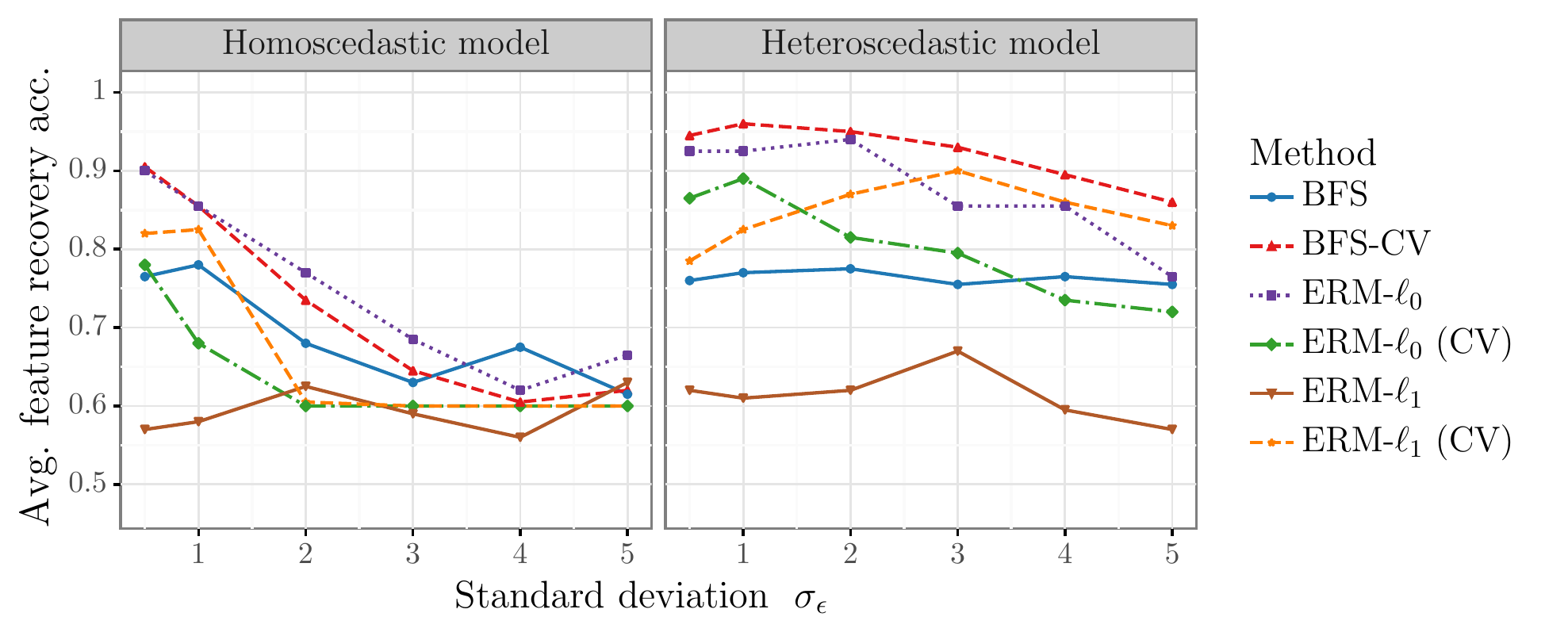}}
	\vspace{-0.2cm}
	\caption{Impact of different noise levels on the accuracy of feature recovery}
	\label{fig:accuracy_feature_recovery-noise-level-nonlinear}
\end{figure}

\begin{figure}[!t]
	\centerline{\includegraphics[width=0.82\linewidth]{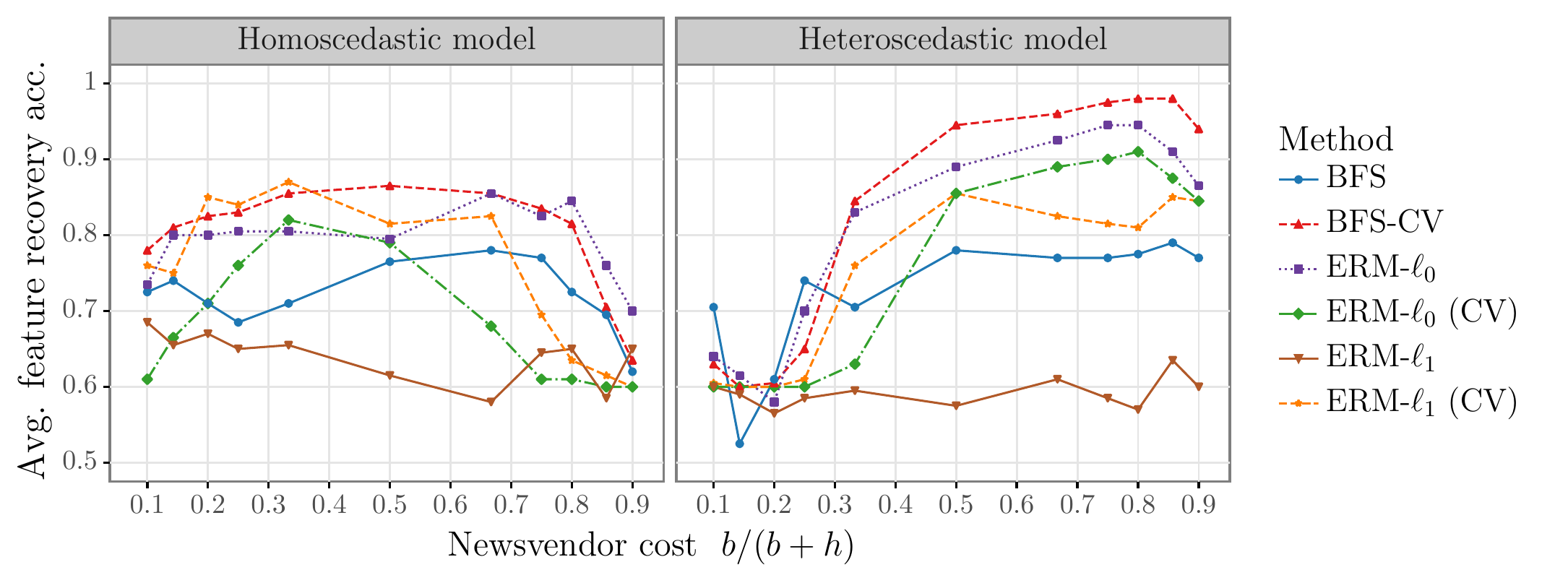}}
	\vspace{-0.2cm}
	\caption{Impact of shortage cost $b$ and holding cost $h$ on the accuracy of feature recovery}
	\label{fig:accuracy_feature_recovery-backorder_cost-nonlinear}
\end{figure}

One comment on these results is in order. Applying a linear decision function to nonlinear data may naturally lead to poor results, due to the inconsistent dependency structure with respect to the features.
In some cases, we observed that all methods fail to achieve reasonable accuracy in feature recovery. 
However, for the vast majority of settings that we considered, \gls{BFS-CV} presented good performance, outperforming the other methods. One possibility for dealing with such nonlinear instances would be to include additional features by considering nonlinear transformations of the original features, which we leave for future research.

\section{Conclusion}\label{section-conclusion}

We have presented a novel formulation based on bilevel optimization for incorporating feature selection in the feature-based newsvendor problem. The proposed \gls{BFS} models provide an intuitive approach specifically designed for the task of feature selection, in which the upper-level problem directly selects the subset of relevant features. Our experimental results on synthetic data show that the proposed methods can accurately recover ground-truth informative features, leading to more explainable inventory decisions in comparison to previous methods. 

There are many possibilities for follow-up works. First, research on tailored solution methods for \gls{BFS} and \gls{BFS-CV}, e.g., based on decomposition strategies for mixed integer programming, may allow to improve the scalability of the proposed methods when dealing with a large number of features. Next, other classes of data-driven optimization problems may benefit from an extension of the proposed methodology. As an example, the newsvendor problem also captures the fundamental trade-offs emerging in decisions related to capacity planning. Accordingly, an extension of the proposed \gls{BFS} models may be applied to select features in stochastic, data-driven variations of such problems. Finally, from a broader perspective, incorporating concepts from machine learning into data-driven optimization problems may lead to crucial advances for both fields. As exemplified in this work, feature selection is one among possibly many machine learning tasks that can be integrated into the decision-making process, especially given the growing need for more explainable decisions. 

\vspace*{0.5cm}
\noindent\textbf{Acknowledgments} 

This research has been funded by the Deutsche Forschungsgemeinschaft (DFG, German Research Foundation) - Projektnummer 277991500. This support is gratefully acknowledged.

\appendix

\section{Computation times}\label{appendix-computation-times}

In this section, we report the computation times of the different methods considered in the experiments.
Figures~\ref{fig:solution_time-features=8}, \ref{fig:solution_time-features=10}, \ref{fig:solution_time-features=12}, and \ref{fig:solution_time-features=14} show how the solution times scale with the number of samples $n$, respectively for $m\in\{8,10,12,14\}$ features.
Here, we consider instances with linear demand model and we fix shortage cost to $b=2$, the holding cost to $h=1$, and the noise level to $\sigma_{\epsilon}=1$.



\begin{figure}[H]
	\centerline{\includegraphics[width=0.9\linewidth]{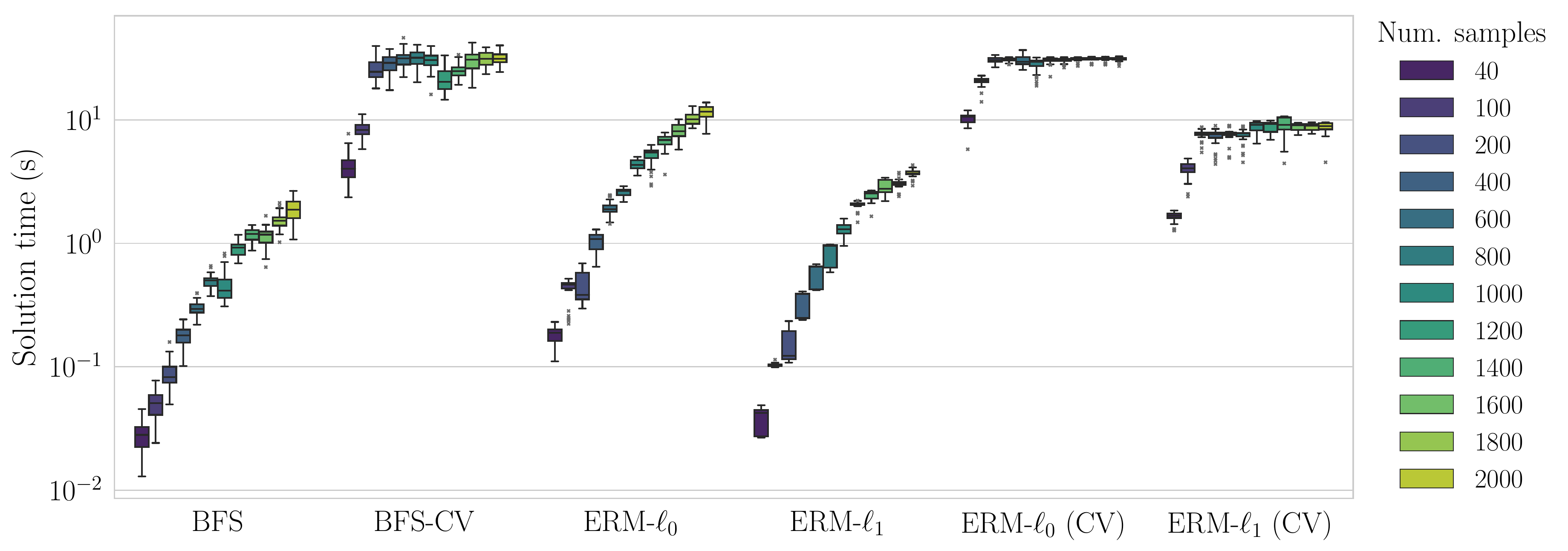}}
	\caption{Solution times (in seconds) for instances with $m=8$ features}
	\label{fig:solution_time-features=8}
\end{figure}

\begin{figure}[H]
	\centerline{\includegraphics[width=0.9\linewidth]{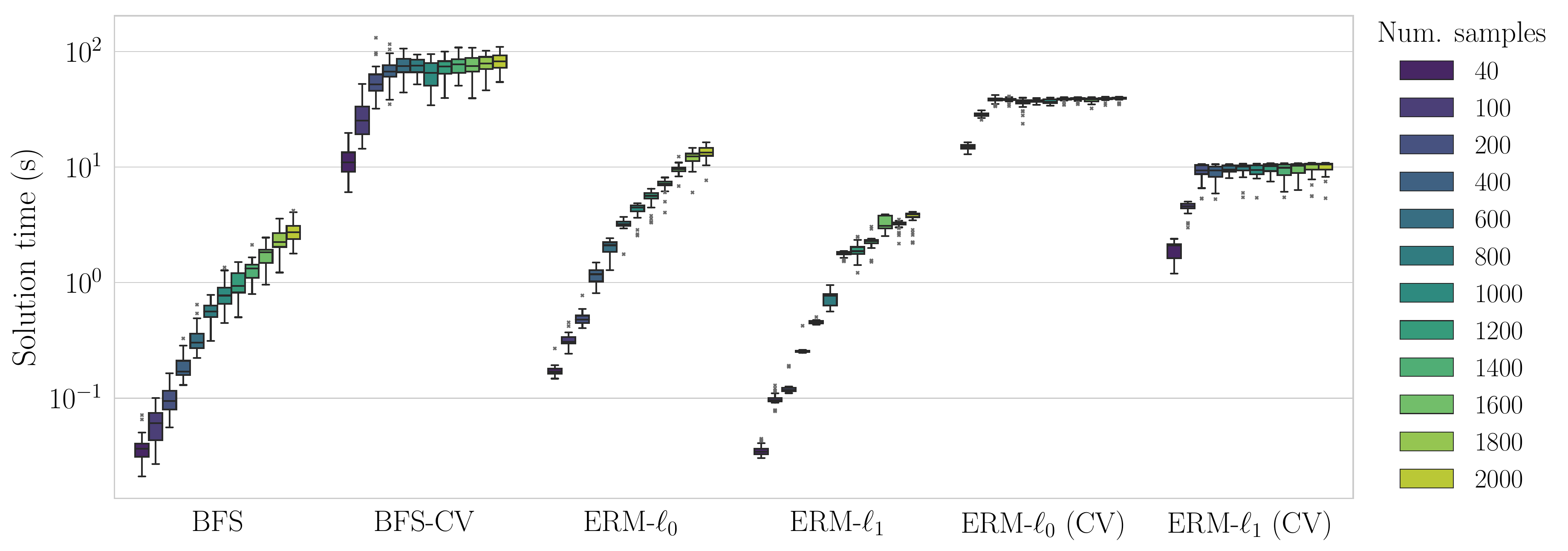}}
	\caption{Solution times (in seconds) for instances with $m=10$ features}
	\label{fig:solution_time-features=10}
\end{figure}

\begin{figure}[H]
	\centerline{\includegraphics[width=0.9\linewidth]{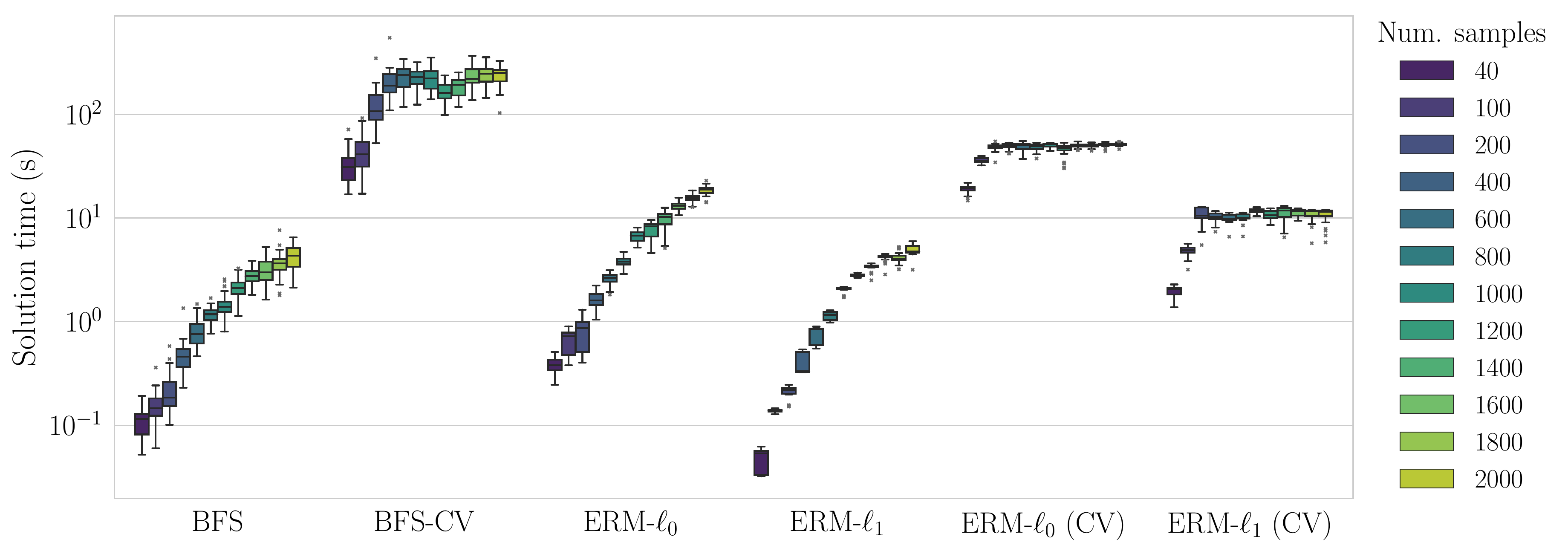}}
	\caption{Solution times (in seconds) for instances with $m=12$ features}
	\label{fig:solution_time-features=12}
\end{figure}
\begin{figure}[H]
	\centerline{\includegraphics[width=0.9\linewidth]{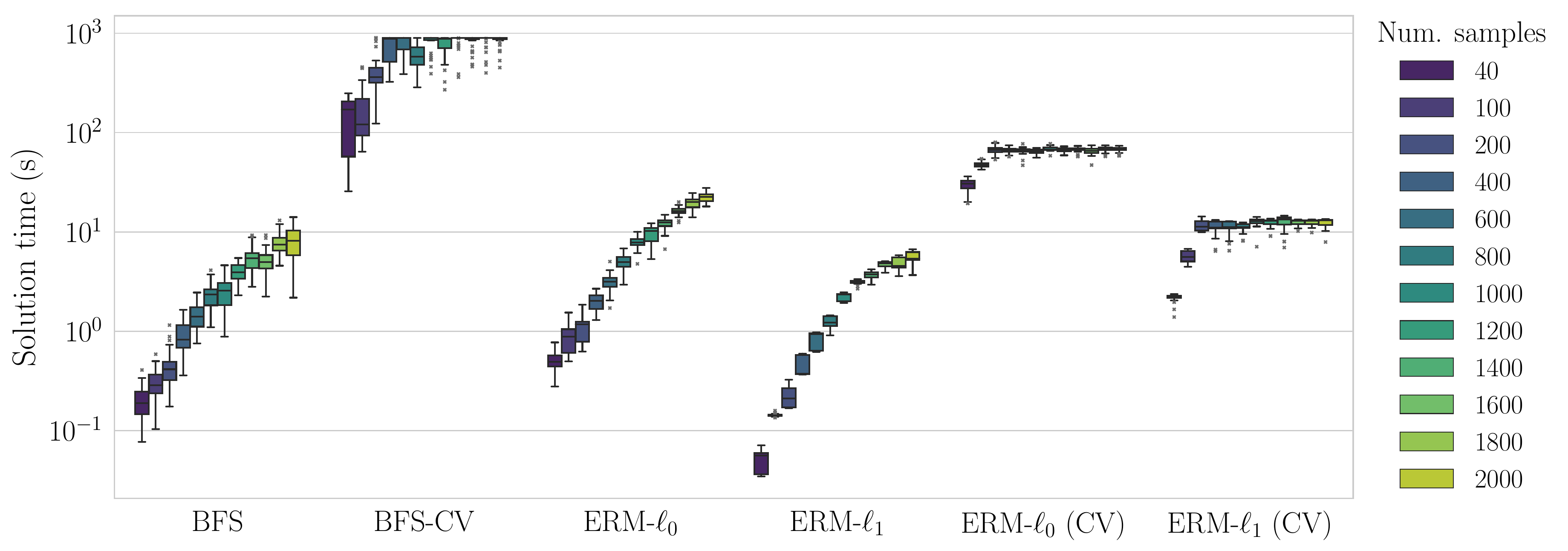}}
	\caption{Solution times (in seconds) for instances with $m=14$ features}
	\label{fig:solution_time-features=14}
\end{figure}

Considering instances with a larger number of features ($m=14$), \gls{BFS-CV} cannot solve all instances to optimality within the specified time limit. 
Accordingly, Figure~\ref{fig:mip_gap-features=14} illustrates the distribution of the optimality gaps (in percentage values) for instances with $m=14$ features. In this case, optimality gaps are often below $5\%$.


\begin{figure}[H]
	\centerline{\includegraphics[width=0.7\linewidth]{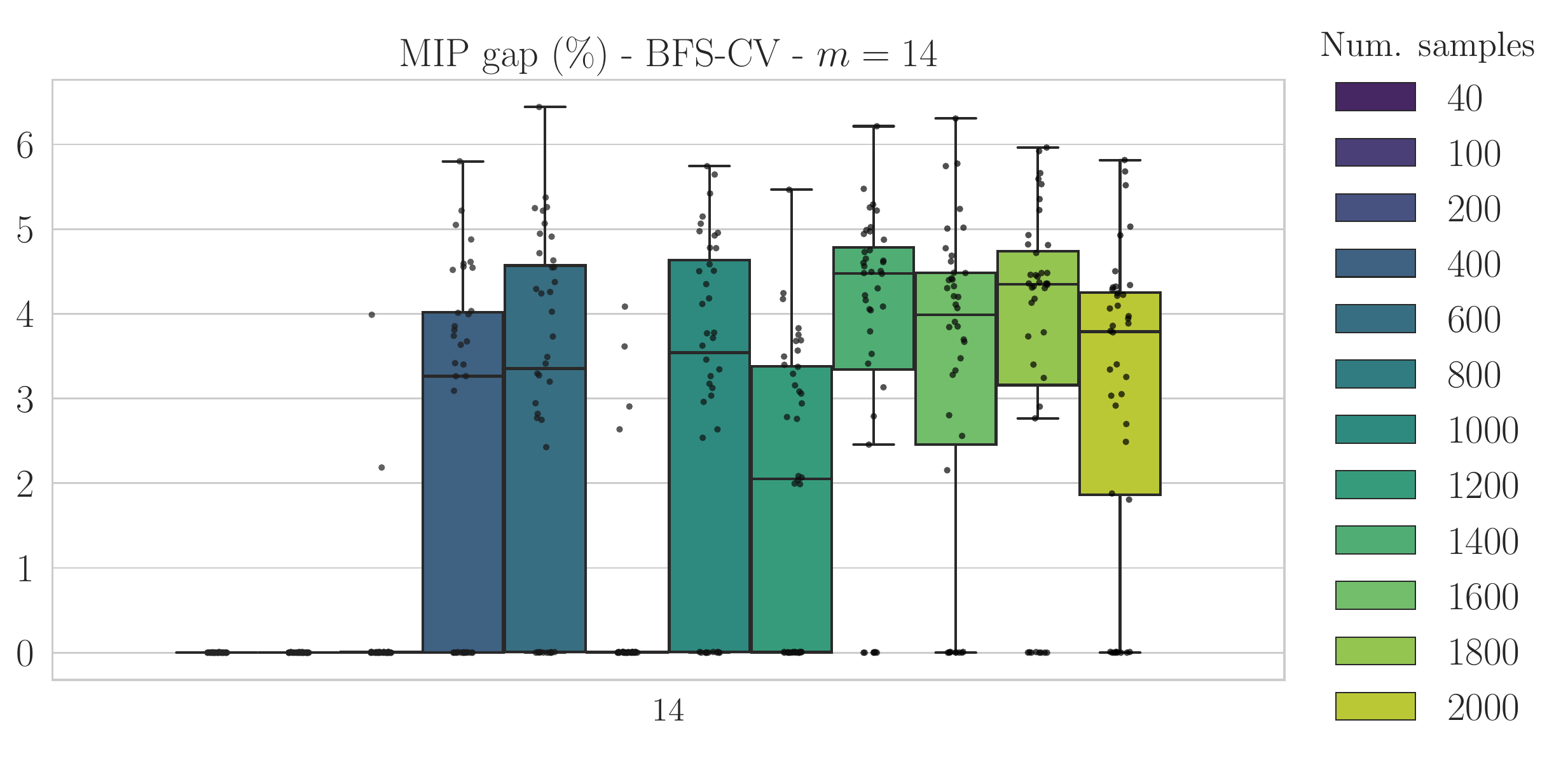}}
	\caption{MIP gaps for $m=14$}
	\label{fig:mip_gap-features=14}
\end{figure}

\section{Complementary results} 
\label{appendix-linear}

In the following, we report the p-values of the non-parametric one-sided Wilcoxon signed-rank test, with null hypothesis that the median difference in accuracy between the BFS-CV and other methods is negative.
Tables~\ref{table:pvalues} and~\ref{table:pvalues-nonlinear} report p-values for instances with linear and nonlinear demand, respectively.

\begin{table}[H]
	\centering
	\setlength{\tabcolsep}{8pt}
	\hspace*{0.0cm}
	{\fontsize{9}{18}\selectfont\begin{tabular}{lccccc}
			\hline
			Figure & BFS &  ERM-$\ell_0$ &  ERM-$\ell_0$ (CV) &  ERM-$\ell_1$ &  ERM-$\ell_1$ (CV) \\
			\hline
			\hline
			Figure~\ref{fig:accuracy_feature_recovery} & 2.31E-35 &  8.73E-12 & 6.88E-23 &      2.25E-38 & 1.23E-32 \\
			\hline
			Figure~\ref{fig:accuracy_feature_recovery-features} & 2.53E-14 & 7.76E-04 & 4.14E-10 & 2.32E-14 & 1.95E-13 \\
			\hline
			Figure~\ref{fig:accuracy_feature_recovery-noise-level} & 3.45E-07 & \textbf{8.89E-02} & 1.08E-11 & 1.06E-15 & 4.47E-13 \\
			\hline
			Figure~\ref{fig:accuracy_feature_recovery-backorder_cost} & 2.00E-32 & 6.56E-08 & 3.96E-23 & 1.77E-37 & 5.02E-27 \\
			\hline
		\end{tabular}
	}
	\caption{P-values for linear instances.
	}
	\label{table:pvalues}
\end{table}

\begin{table}[H]
	\centering
	\setlength{\tabcolsep}{8pt}
	\hspace*{0.0cm}
	{\fontsize{9}{18}\selectfont\begin{tabular}{llccccc}
			\hline
			Figure & Noise & BFS &  ERM-$\ell_0$ &  ERM-$\ell_0$ (CV) &  ERM-$\ell_1$ &  ERM-$\ell_1$ (CV) \\
			\hline
			\hline
			\multirow{2}{*}{Figure~\ref{fig:accuracy_feature_recovery-nonlinear}} & Homoscedastic & 8.20E-33 & \textbf{3.15E-01} & 1.77E-53 & 4.49E-55 & 2.81E-19\\
			\cline{2-7}
			& Heteroscedastic & 8.92E-56 & 1.92E-12 & 1.02E-29 & 7.49E-65 & 9.00E-49 \\
			\hline
			\multirow{2}{*}{Figure~\ref{fig:accuracy_feature_recovery-features-nonlinear}} & Homoscedastic & 4.45E-08 & \textbf{7.87E-01} & 2.66E-13 & 2.01E-12 & 2.37E-05 \\
			\cline{2-7}
			& Heteroscedastic & 6.90E-13 & 3.32E-02 & 8.34E-08 & 9.35E-15 & 2.25E-11 \\
			\hline
			\multirow{2}{*}{Figure~\ref{fig:accuracy_feature_recovery-noise-level-nonlinear}} & Homoscedastic & 2.47E-02 & \textbf{9.86E-01} & 5.60E-11 & 7.19E-10 & 6.27E-07 \\
			\cline{2-7}
			& Heteroscedastic & 3.05E-16 & 1.72E-03 & 9.18E-15 & 5.51E-20 & 6.52E-08 \\
			\hline
			\multirow{2}{*}{Figure~\ref{fig:accuracy_feature_recovery-backorder_cost-nonlinear}} & Homoscedastic & 2.58E-13 & \textbf{3.98E-01} & 6.28E-25 & 2.51E-22 & 2.21E-09 \\
			\cline{2-7}
			& Heteroscedastic & 4.86E-13 & \textbf{6.03E-02} & 4.44E-17 & 1.72E-26 & 1.59E-16 \\
			\hline
		\end{tabular}
	}
	\caption{P-values for nonlinear instances.
	}
	\label{table:pvalues-nonlinear}
\end{table}

\section{Test cost results for nonlinear instances}
\label{appendix-nonlinear}

We present results regarding test cost performance for instances with nonlinear demand model. This section follows the same structure as Section~\ref{section-experiments-nonlinear} of the main paper.

\subsection{Instance size}
\label{appendix-instance-size-nonlinear}

Regarding the impact of instance size $n$, Figures~\ref{fig:relative_test_cost-nb_samples-heteroscedastic-gaussian} and \ref{fig:relative_test_cost-nb_samples-homoscedastic_demand-gaussian_noise} show test cost results for instances with heteroscedastic and homoscedastic demand, respectively.

\begin{figure}[H]
	\hspace{-0.6cm}
	\begin{tabular}{ccc}
		\includegraphics[width=0.36\linewidth]{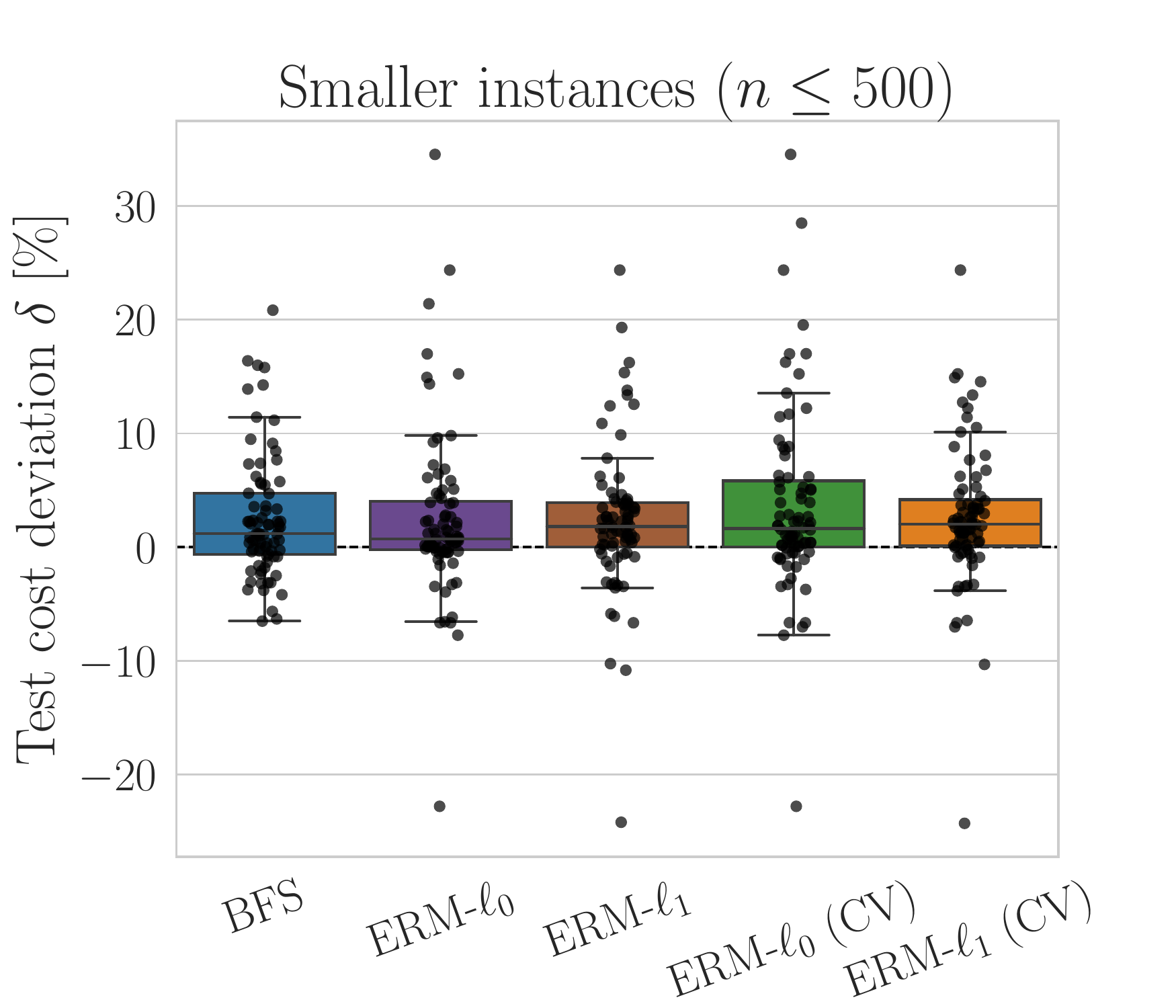} & \hspace{-0.7cm}\includegraphics[width=0.36\linewidth]{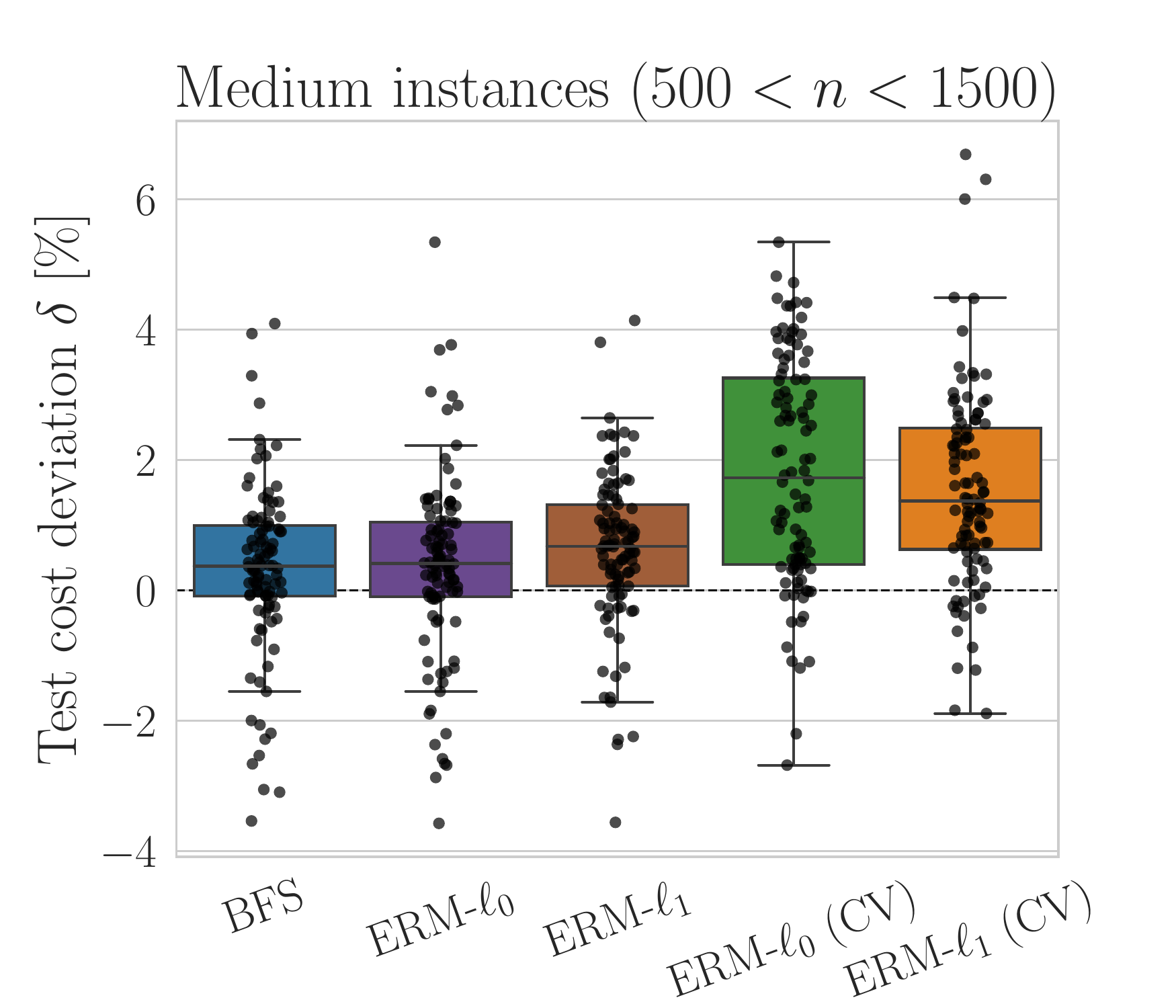} & \hspace{-0.7cm}\includegraphics[width=0.36\linewidth]{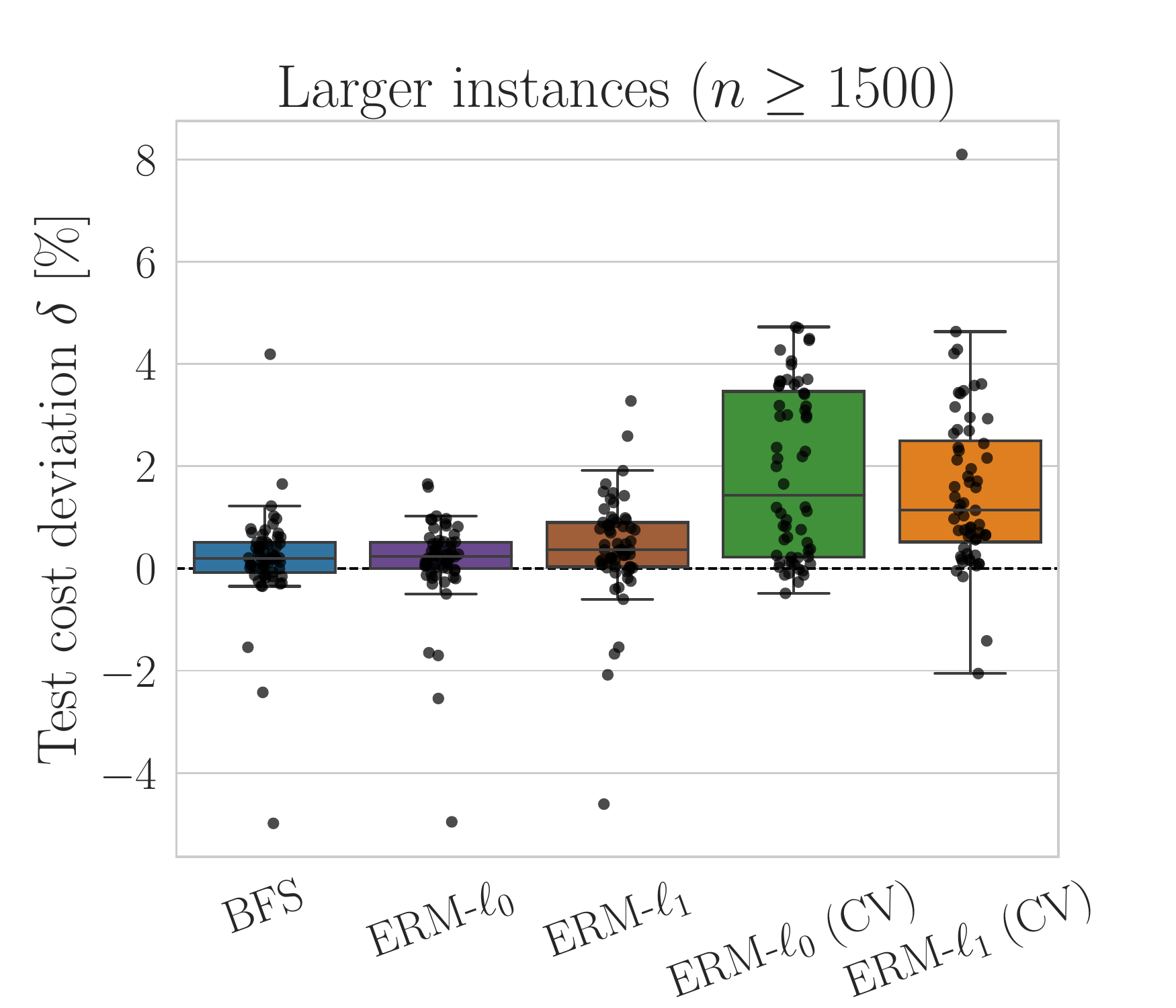} 
	\end{tabular}
	\caption{Impact of instance size on the test cost performance (heteroscedastic demand with Gaussian noise)}
	\label{fig:relative_test_cost-nb_samples-heteroscedastic-gaussian}
\end{figure}

\begin{figure}[H]
	\hspace{-0.6cm}
	\begin{tabular}{ccc}
		\includegraphics[width=0.36\linewidth]{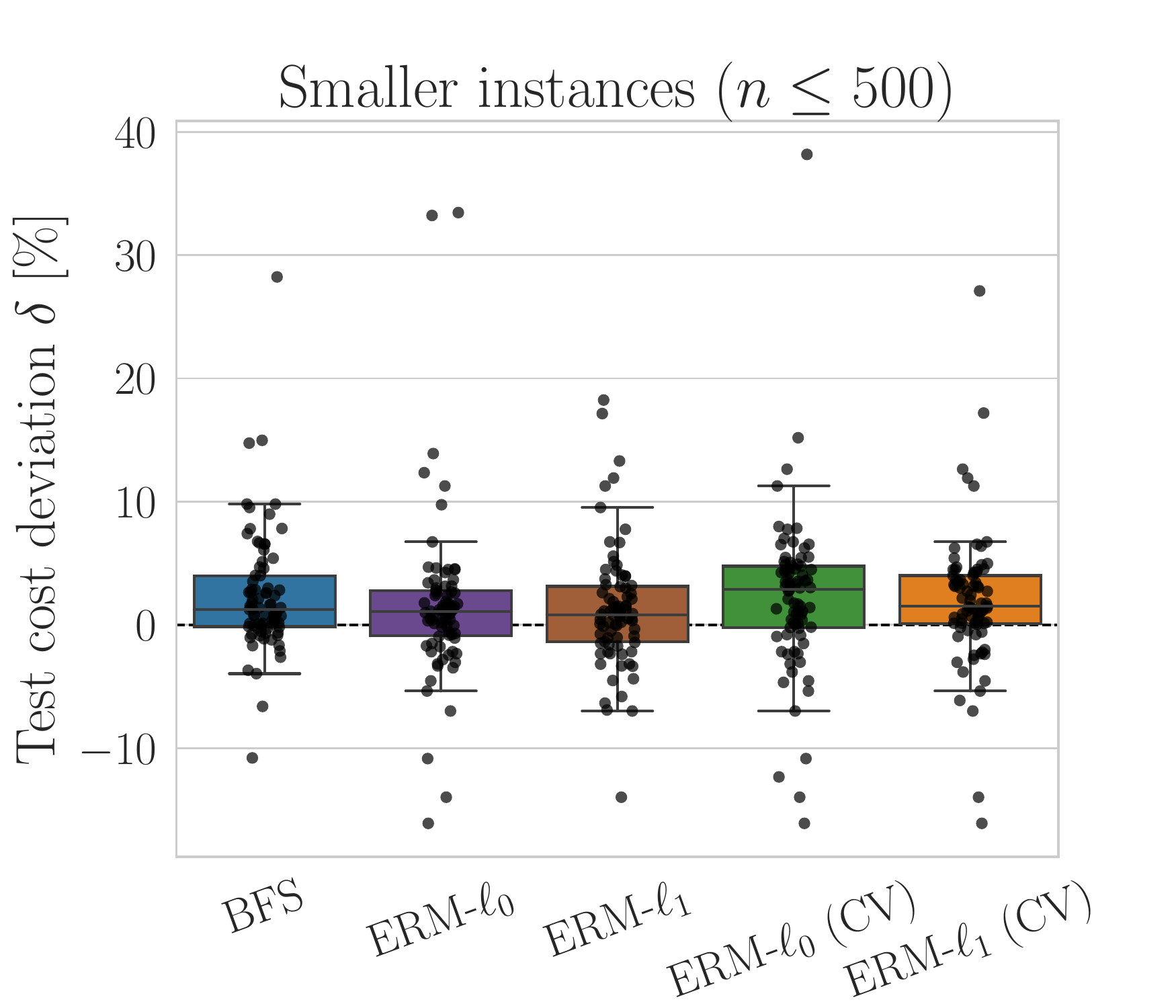} & \hspace{-0.7cm}\includegraphics[width=0.36\linewidth]{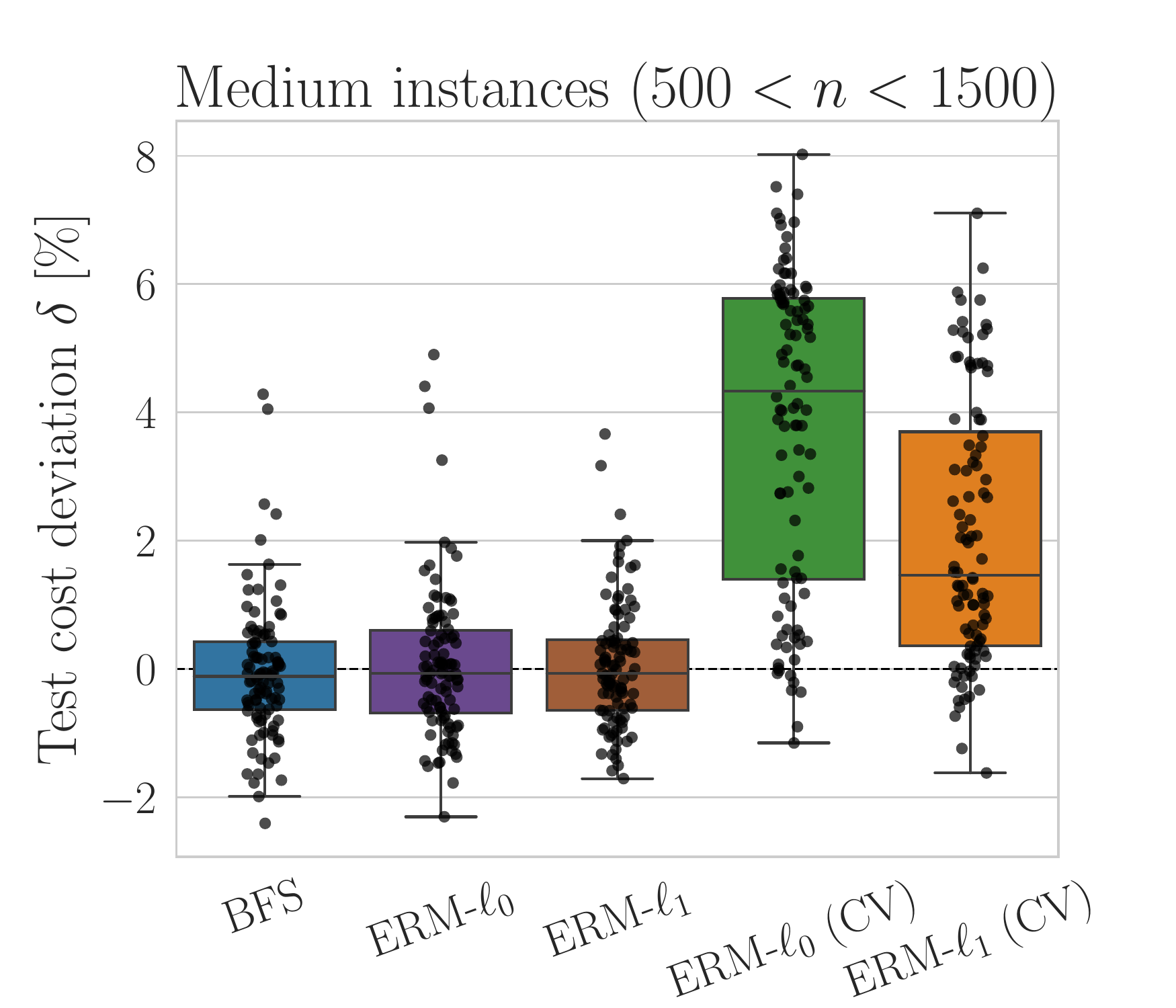} & \hspace{-0.7cm}\includegraphics[width=0.36\linewidth]{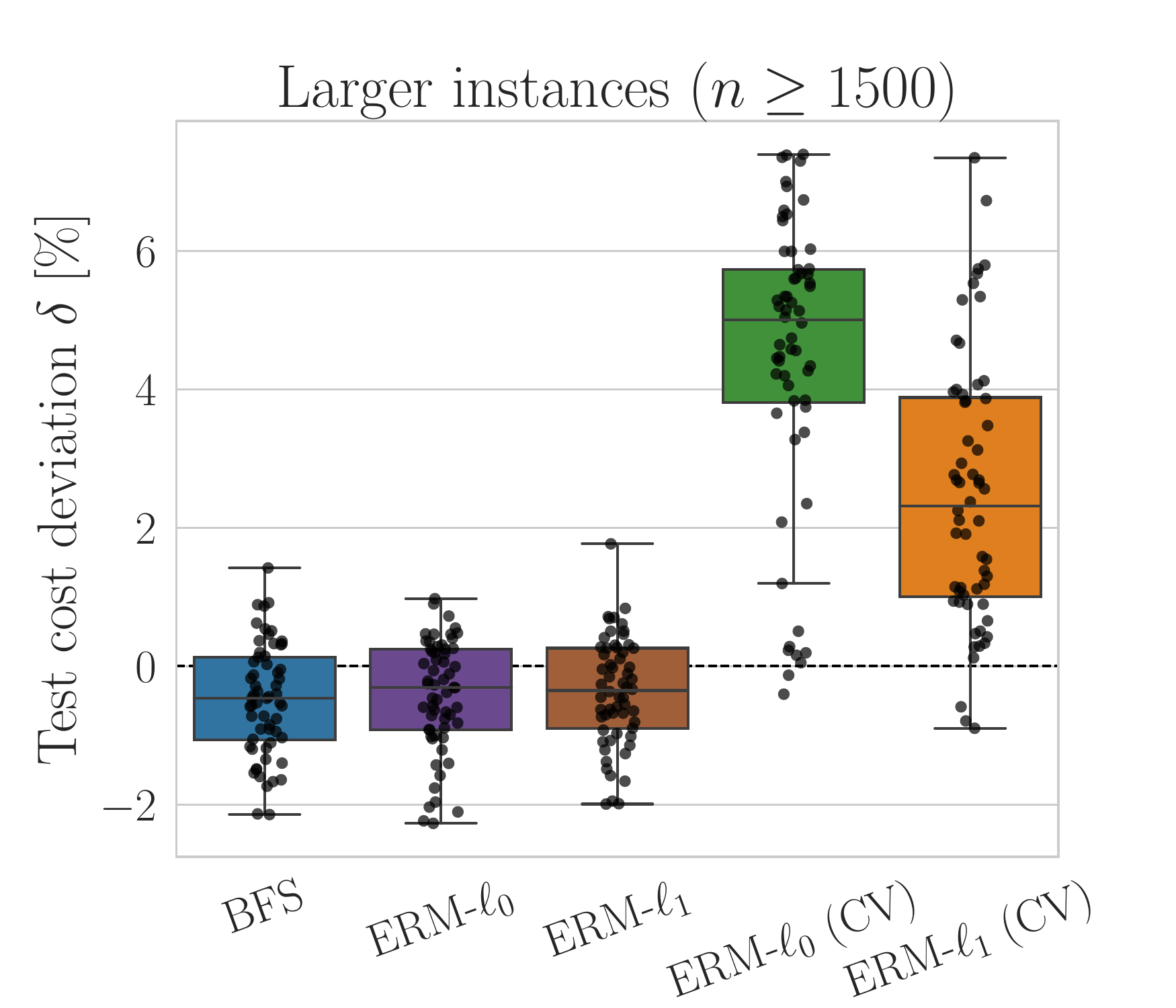} 
	\end{tabular}
	\caption{Impact of instance size on test cost performance (homoscedastic demand with Gaussian noise)}
	\label{fig:relative_test_cost-nb_samples-homoscedastic_demand-gaussian_noise}
\end{figure}

\subsection{Number of features}
\label{appendix-features-nonlinear}

Figures~\ref{fig:relative_test_cost-features-nonlinear} and \ref{fig:relative_test_cost-features-homoscedastic_demand-gaussian_noise} show the impact of the number of features $m$ on test cost results for large instances ($n\geq1500$) with heteroscedastic and homoscedastic demand, respectively.

\begin{figure}[H]
	\centerline{\includegraphics[width=0.8\linewidth]{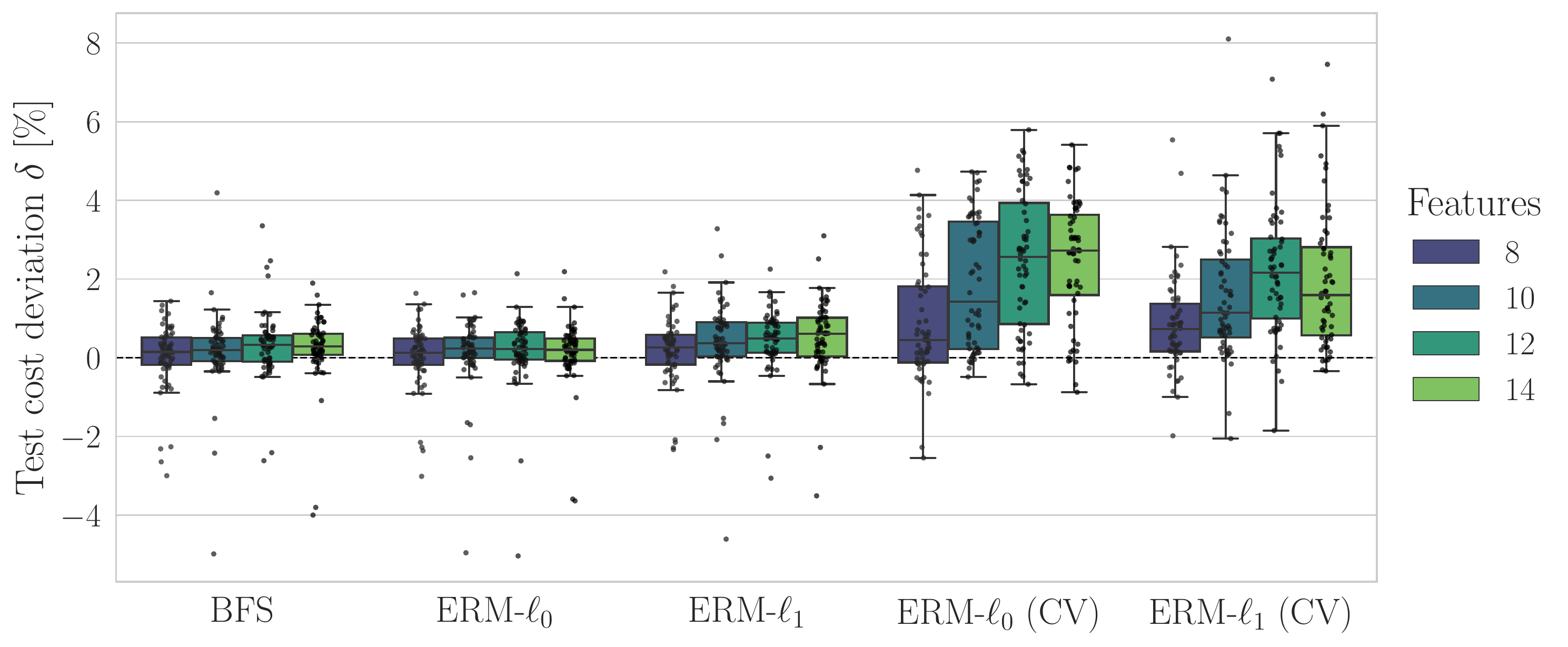}}
	\caption{Impact of number of features $m$ on the test cost (heteroscedastic demand with Gaussian noise)}
	\label{fig:relative_test_cost-features-nonlinear}
\end{figure}

\begin{figure}[H]
	\centerline{\includegraphics[width=0.8\linewidth]{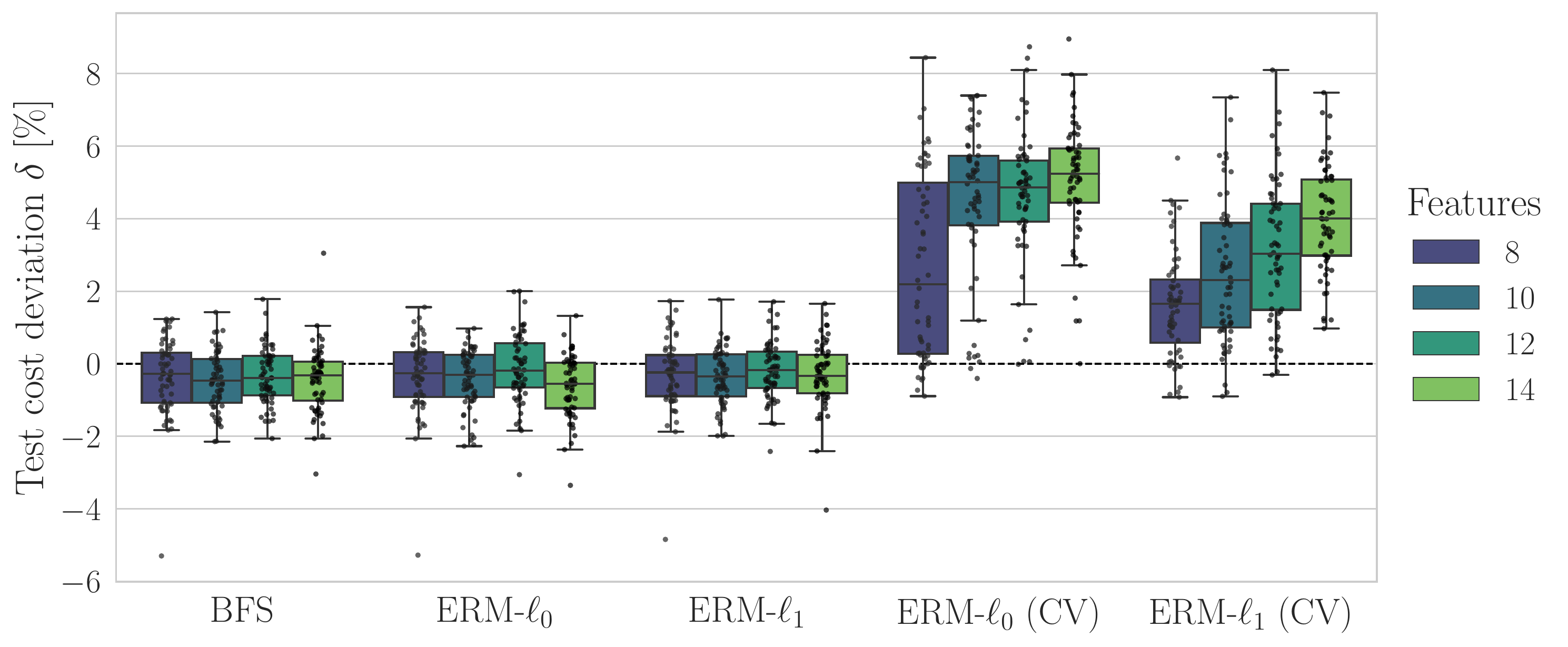}}
	\caption{Impact of number of features $m$ on test cost (homoscedastic demand with Gaussian noise)}
	\label{fig:relative_test_cost-features-homoscedastic_demand-gaussian_noise}
\end{figure}

\subsection{Noise level}
\label{appendix-noise-level-nonlinear}

Figures~\ref{fig:relative_test_cost-noise-nonlinear} and \ref{fig:relative_test_cost-noise_level-homoscedastic_demand-gaussian_noise} show the impact of noise level $\sigma_{\epsilon}$ on test cost results for large instances ($n\geq1500$) with heteroscedastic and homoscedastic demand, respectively.

\begin{figure}[H]
	\centerline{\includegraphics[width=0.8\linewidth]{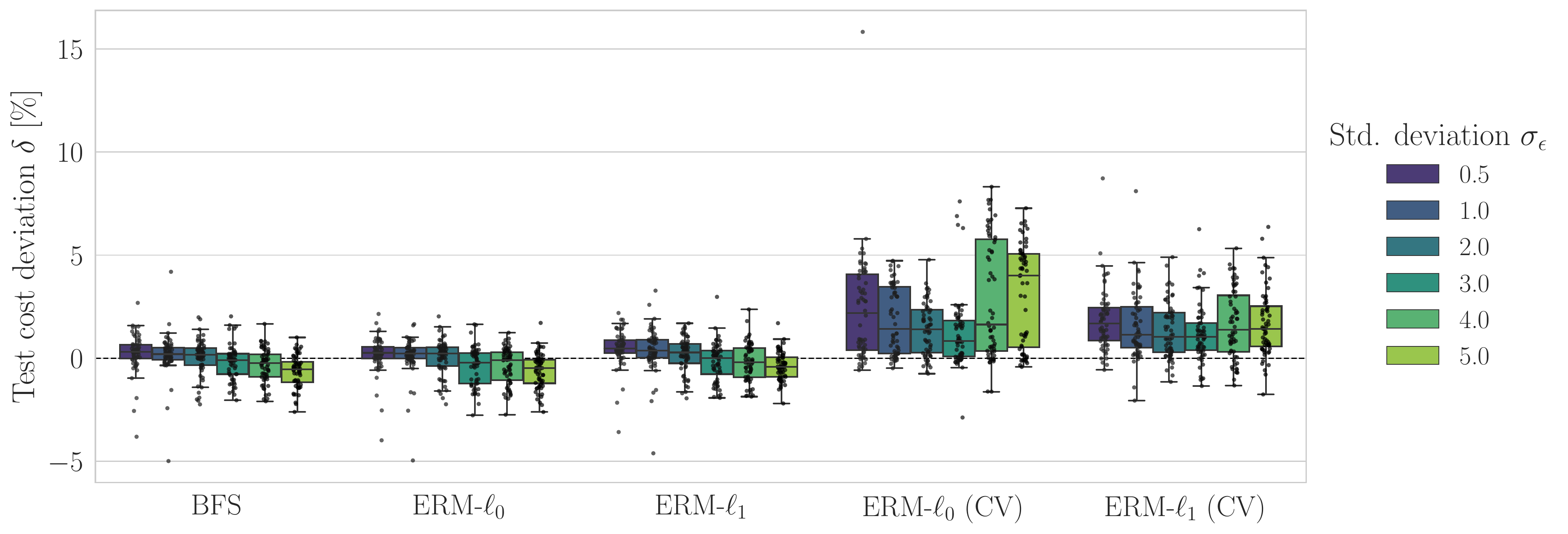}}
	\caption{Impact of noise level $\sigma_{\epsilon}$ on test cost (heteroscedastic demand with Gaussian noise)}
	\label{fig:relative_test_cost-noise-nonlinear}
\end{figure}

\begin{figure}[H]
	\centerline{\includegraphics[width=0.8\linewidth]{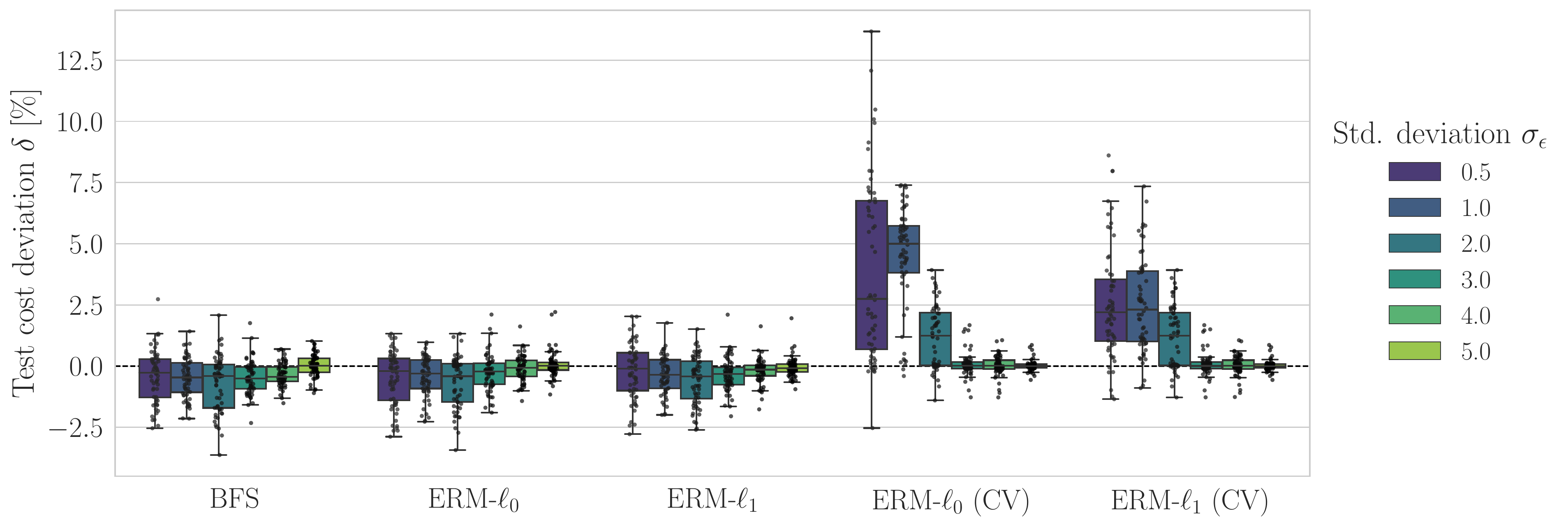}}
	\caption{Impact of noise level $\sigma_{\epsilon}$ on test cost (homoscedastic demand with Gaussian noise)}
	\label{fig:relative_test_cost-noise_level-homoscedastic_demand-gaussian_noise}
\end{figure}

\subsection{Shortage cost}
\label{appendix-backorder-cost-nonlinear}

Figures~\ref{fig:relative_test_cost-backorder_cost-heteroscedastic_demand-gaussian_noise} and \ref{fig:relative_test_cost-backorder_cost-homoscedastic_demand-gaussian_noise} show the impact of shortage cost $b$ on test cost results for large instances ($n\geq1500$) with heteroscedastic and homoscedastic demand, respectively.

\begin{figure}[H]
	\centerline{\includegraphics[width=0.8\linewidth]{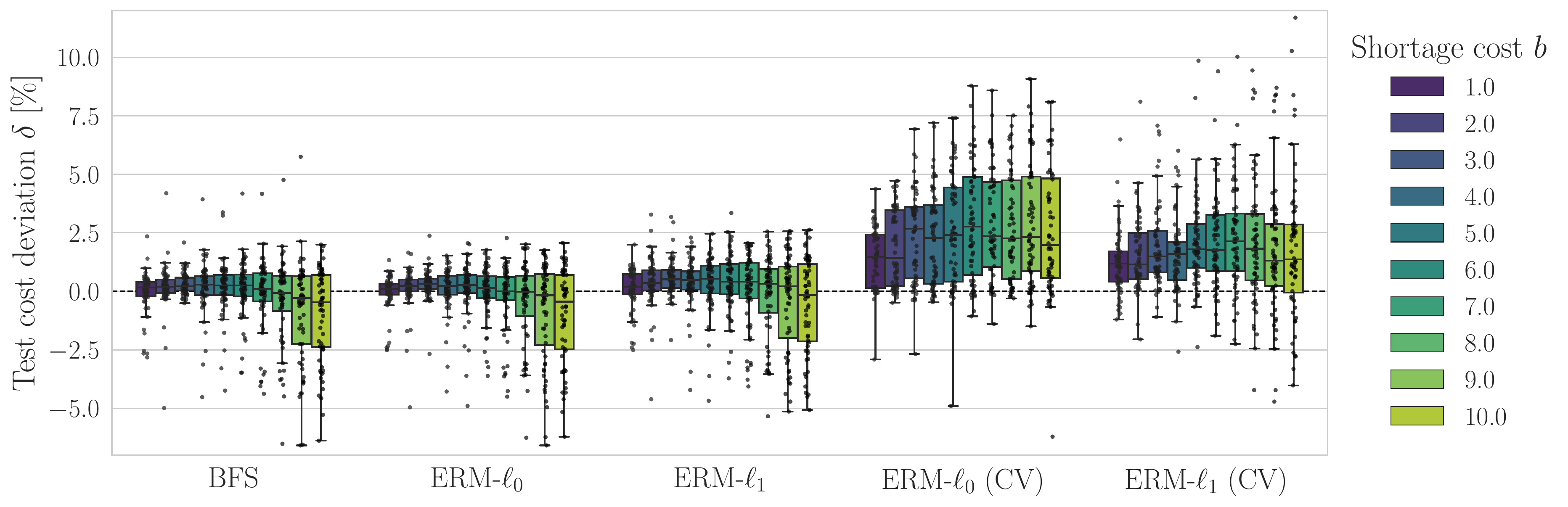}}
	\caption{Impact of shortage cost $b$ on the test cost (heteroscedastic demand with Gaussian noise)}
	\label{fig:relative_test_cost-backorder_cost-heteroscedastic_demand-gaussian_noise}
\end{figure}

\begin{figure}[H]
	\centerline{\includegraphics[width=0.8\linewidth]{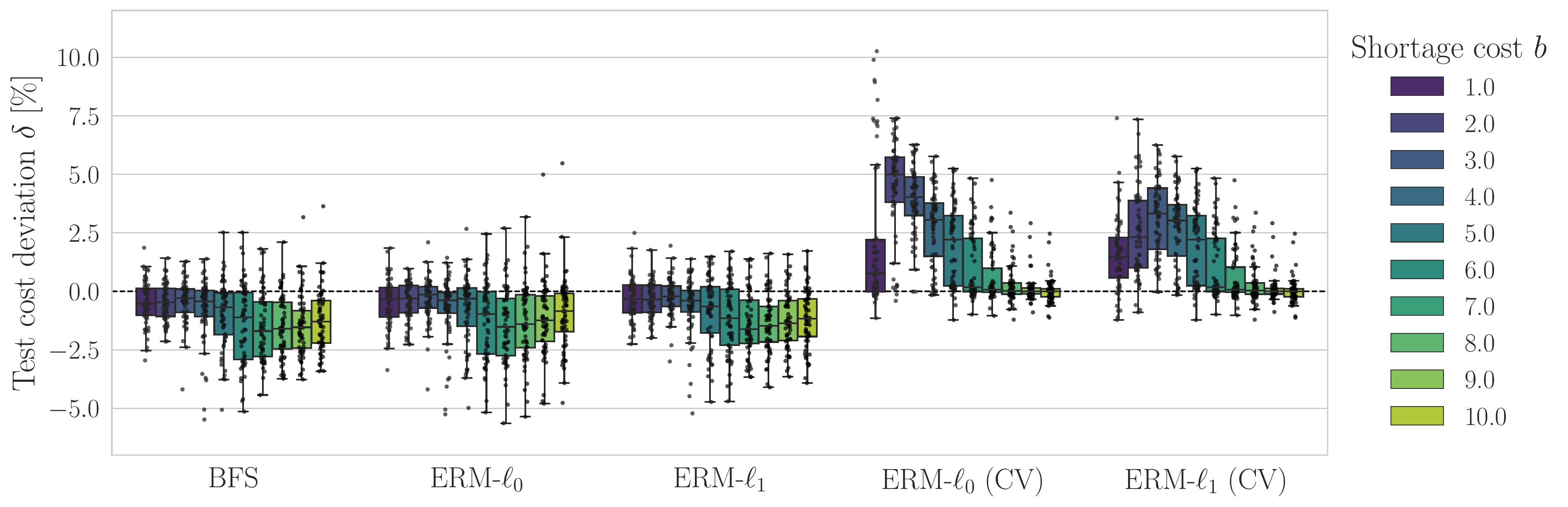}}
	\caption{Impact of shortage cost $b$ on test cost (homoscedastic demand with Gaussian noise)}
	\label{fig:relative_test_cost-backorder_cost-homoscedastic_demand-gaussian_noise}
\end{figure}

\subsection{Holding cost}
\label{appendix-holding-cost-nonlinear}

Figures~\ref{fig:relative_test_cost-holding_cost-heteroscedastic_demand-gaussian_noise} and \ref{fig:relative_test_cost-holding_cost-homoscedastic_demand-gaussian_noise} show the impact of holding cost $h$ on test cost results for large instances ($n\geq1500$) with heteroscedastic and homoscedastic demand, respectively.

\begin{figure}[H]
	\centerline{\includegraphics[width=0.9\linewidth]{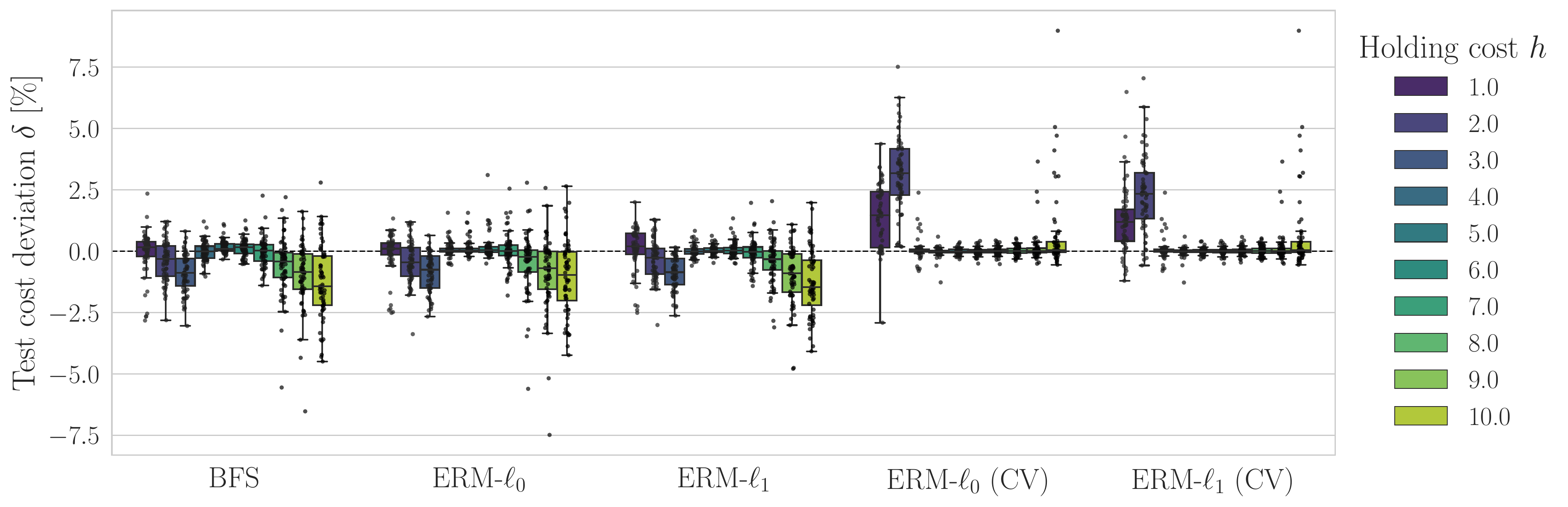}}
	\caption{Impact of holding cost $h$ on test cost (heteroscedastic demand with Gaussian noise)}
	\label{fig:relative_test_cost-holding_cost-heteroscedastic_demand-gaussian_noise}
\end{figure}

\begin{figure}[H]
	\centerline{\includegraphics[width=0.9\linewidth]{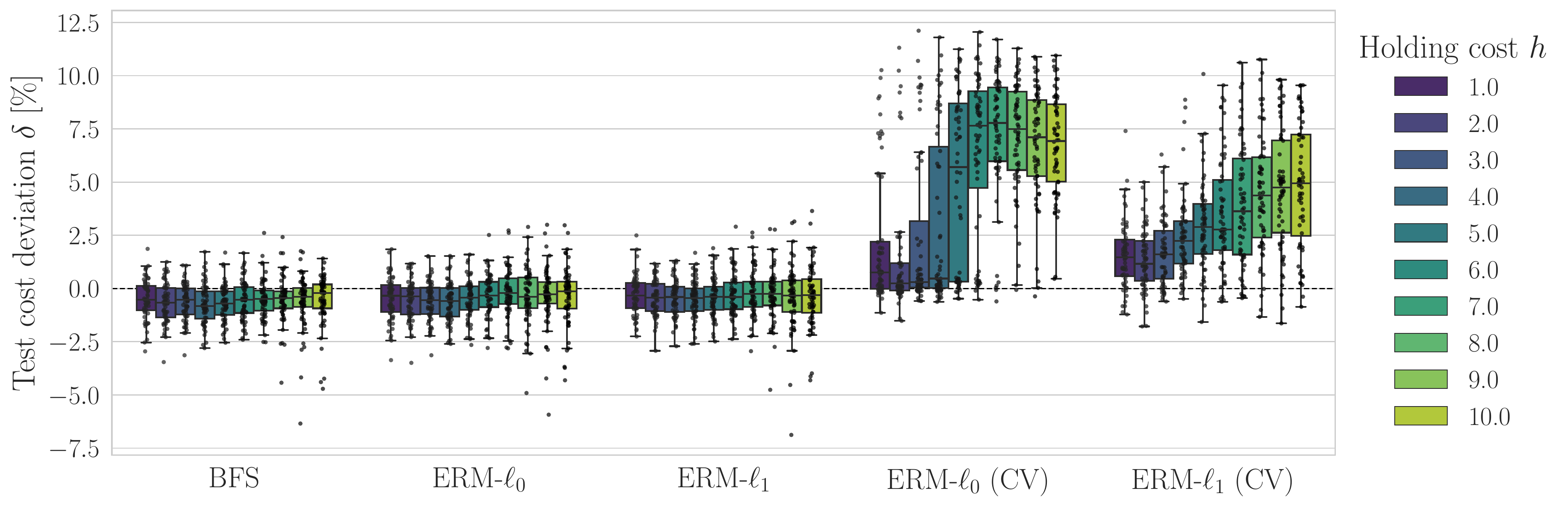}}
	\caption{Impact of holding cost $h$ on test cost (homoscedastic demand with Gaussian noise)}
	\label{fig:relative_test_cost-holding_cost-homoscedastic_demand-gaussian_noise}
\end{figure}


\begin{thebibliography}{53}
	\expandafter\ifx\csname natexlab\endcsname\relax\def\natexlab#1{#1}\fi
	\providecommand{\url}[1]{\texttt{#1}}
	\providecommand{\href}[2]{#2}
	\providecommand{\path}[1]{#1}
	\providecommand{\DOIprefix}{doi:}
	\providecommand{\ArXivprefix}{arXiv:}
	\providecommand{\URLprefix}{URL: }
	\providecommand{\Pubmedprefix}{pmid:}
	\providecommand{\doi}[1]{\href{http://dx.doi.org/#1}{\path{#1}}}
	\providecommand{\Pubmed}[1]{\href{pmid:#1}{\path{#1}}}
	\providecommand{\bibinfo}[2]{#2}
	\ifx\xfnm\relax \def\xfnm[#1]{\unskip,\space#1}\fi
	\bibitem[{Agor \& {\"O}zalt{\i}n(2019)}]{agor2019feature}
	\bibinfo{author}{Agor, J.}, \& \bibinfo{author}{{\"O}zalt{\i}n, O.~Y.}
	(\bibinfo{year}{2019}).
	\newblock \bibinfo{title}{Feature selection for classification models via
		bilevel optimization}.
	\newblock {\it \bibinfo{journal}{Computers \& Operations Research}\/}
	\bibinfo{volume}{106}:\bibinfo{pages}{156--168}.
	\bibitem[{Arlot \& Celisse(2010)}]{arlot2010survey}
	\bibinfo{author}{Arlot, S.}, \& \bibinfo{author}{Celisse, A.}
	(\bibinfo{year}{2010}).
	\newblock \bibinfo{title}{A survey of cross-validation procedures for model
		selection}.
	\newblock {\it \bibinfo{journal}{Statistics Surveys}\/}
	\bibinfo{volume}{4}:\bibinfo{pages}{40--79}.
	\bibitem[{Ban(2020)}]{ban2020confidence}
	\bibinfo{author}{Ban, G.-Y.} (\bibinfo{year}{2020}).
	\newblock \bibinfo{title}{Confidence intervals for data-driven inventory
		policies with demand censoring}.
	\newblock {\it \bibinfo{journal}{Operations Research}\/}
	\bibinfo{volume}{68}(\bibinfo{number}{2}):\bibinfo{pages}{309--326}.
	\bibitem[{Ban \& Rudin(2019)}]{ban2019big}
	\bibinfo{author}{Ban, G.-Y.}, \& \bibinfo{author}{Rudin, C.}
	(\bibinfo{year}{2019}).
	\newblock \bibinfo{title}{The big data newsvendor: {P}ractical insights from
		machine learning}.
	\newblock {\it \bibinfo{journal}{Operations Research}\/}
	\bibinfo{volume}{67}(\bibinfo{number}{1}):\bibinfo{pages}{90--108}.
	\bibitem[{Ben-Tal et~al.(2013)Ben-Tal, Den~Hertog, De~Waegenaere, Melenberg \&
		Rennen}]{ben2013robust}
	\bibinfo{author}{Ben-Tal, A.}, \bibinfo{author}{Den~Hertog, D.},
	\bibinfo{author}{De~Waegenaere, A.}, \bibinfo{author}{Melenberg, B.}, \&
	\bibinfo{author}{Rennen, G.} (\bibinfo{year}{2013}).
	\newblock \bibinfo{title}{Robust solutions of optimization problems affected by
		uncertain probabilities}.
	\newblock {\it \bibinfo{journal}{Management Science}\/}
	\bibinfo{volume}{59}(\bibinfo{number}{2}):\bibinfo{pages}{341--357}.
	\bibitem[{Bennett et~al.(2006)Bennett, Hu, Ji, Kunapuli \&
		Pang}]{bennett2006model}
	\bibinfo{author}{Bennett, K.~P.}, \bibinfo{author}{Hu, J.},
	\bibinfo{author}{Ji, X.}, \bibinfo{author}{Kunapuli, G.}, \&
	\bibinfo{author}{Pang, J.-S.} (\bibinfo{year}{2006}).
	\newblock \bibinfo{title}{Model selection via bilevel optimization}.
	\newblock In {\it \bibinfo{booktitle}{Proceedings of the IEEE International
			Joint Conference on Neural Networks (IJCNN)}\/}.
	\newblock (pp. \bibinfo{pages}{1922--1929}).
	\bibitem[{Bennett et~al.(2008)Bennett, Kunapuli, Hu \&
		Pang}]{bennett2008bilevel}
	\bibinfo{author}{Bennett, K.~P.}, \bibinfo{author}{Kunapuli, G.},
	\bibinfo{author}{Hu, J.}, \& \bibinfo{author}{Pang, J.-S.}
	(\bibinfo{year}{2008}).
	\newblock \bibinfo{title}{Bilevel optimization and machine learning}.
	\newblock In \bibinfo{editor}{J.~M. Zurada}, \bibinfo{editor}{G.~G. Yen}, \&
	\bibinfo{editor}{J.~Wang} (Eds.), {\it \bibinfo{booktitle}{Computational
			Intelligence: Research Frontiers: IEEE World Congress on Computational
			Intelligence (WCCI), Hong Kong, China}\/}.
	\newblock \bibinfo{publisher}{Springer, Berlin, Heidelberg} (pp.
	\bibinfo{pages}{25--47}).
	\bibitem[{Bergstra \& Bengio(2012)}]{bergstra2012random}
	\bibinfo{author}{Bergstra, J.}, \& \bibinfo{author}{Bengio, Y.}
	(\bibinfo{year}{2012}).
	\newblock \bibinfo{title}{Random search for hyper-parameter optimization}.
	\newblock {\it \bibinfo{journal}{Journal of Machine Learning Research}\/}
	\bibinfo{volume}{13}(\bibinfo{number}{10}):\bibinfo{pages}{281--305}.
	\bibitem[{Bergstra et~al.(2013)Bergstra, Yamins \& Cox}]{pmlr-v28-bergstra13}
	\bibinfo{author}{Bergstra, J.}, \bibinfo{author}{Yamins, D.}, \&
	\bibinfo{author}{Cox, D.} (\bibinfo{year}{2013}).
	\newblock \bibinfo{title}{Making a science of model search: Hyperparameter
		optimization in hundreds of dimensions for vision architectures}.
	\newblock In \bibinfo{editor}{S.~Dasgupta}, \& \bibinfo{editor}{D.~McAllester}
	(Eds.), {\it \bibinfo{booktitle}{Proceedings of the 30th International
			Conference on Machine Learning}\/}.
	\newblock \bibinfo{address}{Atlanta, Georgia, USA}: \bibinfo{publisher}{PMLR},
	volume~\bibinfo{volume}{28} of {\it \bibinfo{series}{Proceedings of Machine
			Learning Research}\/} (pp. \bibinfo{pages}{115--123}).
	\bibitem[{Bertsimas et~al.(2018)Bertsimas, Gupta \&
		Kallus}]{bertsimas2018robust}
	\bibinfo{author}{Bertsimas, D.}, \bibinfo{author}{Gupta, V.}, \&
	\bibinfo{author}{Kallus, N.} (\bibinfo{year}{2018}).
	\newblock \bibinfo{title}{Robust sample average approximation}.
	\newblock {\it \bibinfo{journal}{Mathematical Programming}\/}
	\bibinfo{volume}{171}(\bibinfo{number}{1}):\bibinfo{pages}{217--282}.
	\bibitem[{Bertsimas \& Kallus(2020)}]{bertsimas2020predictive}
	\bibinfo{author}{Bertsimas, D.}, \& \bibinfo{author}{Kallus, N.}
	(\bibinfo{year}{2020}).
	\newblock \bibinfo{title}{From predictive to prescriptive analytics}.
	\newblock {\it \bibinfo{journal}{Management Science}\/}
	\bibinfo{volume}{66}(\bibinfo{number}{3}):\bibinfo{pages}{1025--1044}.
	\bibitem[{Bertsimas \& Thiele(2005)}]{bertsimas2005data}
	\bibinfo{author}{Bertsimas, D.}, \& \bibinfo{author}{Thiele, A.}
	(\bibinfo{year}{2005}).
	\newblock {\it \bibinfo{title}{A data-driven approach to newsvendor
			problems}\/}.
	\newblock \bibinfo{type}{Technical Report,} Massachusetts Institute of
	Technology \bibinfo{address}{Cambridge, MA}.
	\bibitem[{Bertsimas \& Thiele(2006)}]{bertsimas2006robust}
	\bibinfo{author}{Bertsimas, D.}, \& \bibinfo{author}{Thiele, A.}
	(\bibinfo{year}{2006}).
	\newblock \bibinfo{title}{A robust optimization approach to inventory theory}.
	\newblock {\it \bibinfo{journal}{Operations Research}\/}
	\bibinfo{volume}{54}(\bibinfo{number}{1}):\bibinfo{pages}{150--168}.
	\bibitem[{Besbes \& Mouchtaki(2021)}]{besbes2021big}
	\bibinfo{author}{Besbes, O.}, \& \bibinfo{author}{Mouchtaki, O.}
	(\bibinfo{year}{2021}).
	\newblock \bibinfo{title}{How big should your data really be? {D}ata-driven
		newsvendor: {L}earning one sample at a time}.
	\newblock {\it \bibinfo{journal}{Available at SSRN}\/}
	\bibinfo{volume}{3878155}.
	\bibitem[{Besbes \& Muharremoglu(2013)}]{besbes2013implications}
	\bibinfo{author}{Besbes, O.}, \& \bibinfo{author}{Muharremoglu, A.}
	(\bibinfo{year}{2013}).
	\newblock \bibinfo{title}{On implications of demand censoring in the newsvendor
		problem}.
	\newblock {\it \bibinfo{journal}{Management Science}\/}
	\bibinfo{volume}{59}(\bibinfo{number}{6}):\bibinfo{pages}{1407--1424}.
	\bibitem[{Beutel \& Minner(2012)}]{beutel2012safety}
	\bibinfo{author}{Beutel, A.-L.}, \& \bibinfo{author}{Minner, S.}
	(\bibinfo{year}{2012}).
	\newblock \bibinfo{title}{Safety stock planning under causal demand
		forecasting}.
	\newblock {\it \bibinfo{journal}{International Journal of Production
			Economics}\/}
	\bibinfo{volume}{140}(\bibinfo{number}{2}):\bibinfo{pages}{637--645}.
	\bibitem[{Cao \& Chen(2006)}]{cao2006capacitated}
	\bibinfo{author}{Cao, D.}, \& \bibinfo{author}{Chen, M.}
	(\bibinfo{year}{2006}).
	\newblock \bibinfo{title}{Capacitated plant selection in a decentralized
		manufacturing environment: {A} bilevel optimization approach}.
	\newblock {\it \bibinfo{journal}{European Journal of Operational Research}\/}
	\bibinfo{volume}{169}(\bibinfo{number}{1}):\bibinfo{pages}{97--110}.
	\bibitem[{Chen et~al.(2007)Chen, Sim, Simchi-Levi \& Sun}]{chen2007risk}
	\bibinfo{author}{Chen, X.}, \bibinfo{author}{Sim, M.},
	\bibinfo{author}{Simchi-Levi, D.}, \& \bibinfo{author}{Sun, P.}
	(\bibinfo{year}{2007}).
	\newblock \bibinfo{title}{Risk aversion in inventory management}.
	\newblock {\it \bibinfo{journal}{Operations Research}\/}
	\bibinfo{volume}{55}(\bibinfo{number}{5}):\bibinfo{pages}{828--842}.
	\bibitem[{Cheung \& Simchi-Levi(2019)}]{cheung2019sampling}
	\bibinfo{author}{Cheung, W.~C.}, \& \bibinfo{author}{Simchi-Levi, D.}
	(\bibinfo{year}{2019}).
	\newblock \bibinfo{title}{Sampling-based approximation schemes for capacitated
		stochastic inventory control models}.
	\newblock {\it \bibinfo{journal}{Mathematics of Operations Research}\/}
	\bibinfo{volume}{44}(\bibinfo{number}{2}):\bibinfo{pages}{668--692}.
	\bibitem[{Choi(2012)}]{choi2012handbook}
	\bibinfo{author}{Choi, T.-M.} (\bibinfo{year}{2012}).
	\newblock {\it \bibinfo{title}{Handbook of newsvendor problems: {M}odels,
			extensions and applications}\/}, volume \bibinfo{volume}{176}.
	\newblock \bibinfo{publisher}{Springer, New York}.
	\bibitem[{Elmachtoub \& Grigas(2021)}]{elmachtoub2021smart}
	\bibinfo{author}{Elmachtoub, A.~N.}, \& \bibinfo{author}{Grigas, P.}
	(\bibinfo{year}{2021}).
	\newblock \bibinfo{title}{Smart “predict, then optimize”}.
	\newblock {\it \bibinfo{journal}{Management Science}\/}
	\bibinfo{volume}{68}(\bibinfo{number}{1}):\bibinfo{pages}{9--26}.
	\bibitem[{Fletcher \& Leyffer(2002)}]{fletcher2002nonlinear}
	\bibinfo{author}{Fletcher, R.}, \& \bibinfo{author}{Leyffer, S.}
	(\bibinfo{year}{2002}).
	\newblock \bibinfo{title}{Nonlinear programming without a penalty function}.
	\newblock {\it \bibinfo{journal}{Mathematical Programming}\/}
	\bibinfo{volume}{91}(\bibinfo{number}{2}):\bibinfo{pages}{239--269}.
	\bibitem[{Fontaine \& Minner(2014)}]{fontaine2014benders}
	\bibinfo{author}{Fontaine, P.}, \& \bibinfo{author}{Minner, S.}
	(\bibinfo{year}{2014}).
	\newblock \bibinfo{title}{Benders decomposition for discrete--continuous linear
		bilevel problems with application to traffic network design}.
	\newblock {\it \bibinfo{journal}{Transportation Research Part B:
			Methodological}\/} \bibinfo{volume}{70}:\bibinfo{pages}{163--172}.
	\bibitem[{Franceschi et~al.(2018)Franceschi, Frasconi, Salzo, Grazzi \&
		Pontil}]{franceschi2018bilevel}
	\bibinfo{author}{Franceschi, L.}, \bibinfo{author}{Frasconi, P.},
	\bibinfo{author}{Salzo, S.}, \bibinfo{author}{Grazzi, R.}, \&
	\bibinfo{author}{Pontil, M.} (\bibinfo{year}{2018}).
	\newblock \bibinfo{title}{Bilevel programming for hyperparameter optimization
		and meta-learning}.
	\newblock In \bibinfo{editor}{J.~Dy}, \& \bibinfo{editor}{A.~Krause} (Eds.),
	{\it \bibinfo{booktitle}{Proceedings of the 35th International Conference on
			Machine Learning}\/}.
	\newblock \bibinfo{address}{Stockholm, Sweden}: \bibinfo{publisher}{PMLR},
	volume~\bibinfo{volume}{80} of {\it \bibinfo{series}{Proceedings of Machine
			Learning Research}\/} (pp. \bibinfo{pages}{1568--1577}).
	\bibitem[{Gallego \& Moon(1993)}]{gallego1993distribution}
	\bibinfo{author}{Gallego, G.}, \& \bibinfo{author}{Moon, I.}
	(\bibinfo{year}{1993}).
	\newblock \bibinfo{title}{The distribution free newsboy problem: {R}eview and
		extensions}.
	\newblock {\it \bibinfo{journal}{Journal of the Operational Research
			Society}\/}
	\bibinfo{volume}{44}(\bibinfo{number}{8}):\bibinfo{pages}{825--834}.
	\bibitem[{Gallego et~al.(2001)Gallego, Ryan \&
		Simchi-Levi}]{gallego2001minimax}
	\bibinfo{author}{Gallego, G.}, \bibinfo{author}{Ryan, J.~K.}, \&
	\bibinfo{author}{Simchi-Levi, D.} (\bibinfo{year}{2001}).
	\newblock \bibinfo{title}{Minimax analysis for finite-horizon inventory
		models}.
	\newblock {\it \bibinfo{journal}{IIE Transactions}\/}
	\bibinfo{volume}{33}(\bibinfo{number}{10}):\bibinfo{pages}{861--874}.
	\bibitem[{Hanasusanto et~al.(2015)Hanasusanto, Kuhn, Wallace \&
		Zymler}]{hanasusanto2015distributionally}
	\bibinfo{author}{Hanasusanto, G.~A.}, \bibinfo{author}{Kuhn, D.},
	\bibinfo{author}{Wallace, S.~W.}, \& \bibinfo{author}{Zymler, S.}
	(\bibinfo{year}{2015}).
	\newblock \bibinfo{title}{Distributionally robust multi-item newsvendor
		problems with multimodal demand distributions}.
	\newblock {\it \bibinfo{journal}{Mathematical Programming}\/}
	\bibinfo{volume}{152}(\bibinfo{number}{1}):\bibinfo{pages}{1--32}.
	\bibitem[{Hastie et~al.(2009)Hastie, Tibshirani, Friedman \&
		Friedman}]{hastie2009elements}
	\bibinfo{author}{Hastie, T.}, \bibinfo{author}{Tibshirani, R.},
	\bibinfo{author}{Friedman, J.~H.}, \& \bibinfo{author}{Friedman, J.~H.}
	(\bibinfo{year}{2009}).
	\newblock {\it \bibinfo{title}{The elements of statistical learning: {D}ata
			mining, inference, and prediction}\/}, volume~\bibinfo{volume}{2}.
	\newblock \bibinfo{publisher}{Springer, New York}.
	\bibitem[{Huber et~al.(2019)Huber, M{\"u}ller, Fleischmann \&
		Stuckenschmidt}]{huber2019data}
	\bibinfo{author}{Huber, J.}, \bibinfo{author}{M{\"u}ller, S.},
	\bibinfo{author}{Fleischmann, M.}, \& \bibinfo{author}{Stuckenschmidt, H.}
	(\bibinfo{year}{2019}).
	\newblock \bibinfo{title}{A data-driven newsvendor problem: {F}rom data to
		decision}.
	\newblock {\it \bibinfo{journal}{European Journal of Operational Research}\/}
	\bibinfo{volume}{278}(\bibinfo{number}{3}):\bibinfo{pages}{904--915}.
	\bibitem[{Khouja(1999)}]{khouja1999single}
	\bibinfo{author}{Khouja, M.} (\bibinfo{year}{1999}).
	\newblock \bibinfo{title}{The single-period (news-vendor) problem: {L}iterature
		review and suggestions for future research}.
	\newblock {\it \bibinfo{journal}{Omega}\/}
	\bibinfo{volume}{27}(\bibinfo{number}{5}):\bibinfo{pages}{537--553}.
	\bibitem[{Kleywegt et~al.(2002)Kleywegt, Shapiro \& Homem-de
		Mello}]{kleywegt2002sample}
	\bibinfo{author}{Kleywegt, A.~J.}, \bibinfo{author}{Shapiro, A.}, \&
	\bibinfo{author}{Homem-de Mello, T.} (\bibinfo{year}{2002}).
	\newblock \bibinfo{title}{The sample average approximation method for
		stochastic discrete optimization}.
	\newblock {\it \bibinfo{journal}{SIAM Journal on Optimization}\/}
	\bibinfo{volume}{12}(\bibinfo{number}{2}):\bibinfo{pages}{479--502}.
	\bibitem[{Kogan \& Lou(2003)}]{kogan2003multi}
	\bibinfo{author}{Kogan, K.}, \& \bibinfo{author}{Lou, S.}
	(\bibinfo{year}{2003}).
	\newblock \bibinfo{title}{Multi-stage newsboy problem: {A} dynamic model}.
	\newblock {\it \bibinfo{journal}{European Journal of Operational Research}\/}
	\bibinfo{volume}{149}(\bibinfo{number}{2}):\bibinfo{pages}{448--458}.
	\bibitem[{Kuhn \& Johnson(2019)}]{kuhn2019feature}
	\bibinfo{author}{Kuhn, M.}, \& \bibinfo{author}{Johnson, K.}
	(\bibinfo{year}{2019}).
	\newblock {\it \bibinfo{title}{Feature engineering and selection: {A} practical
			approach for predictive models}\/}.
	\newblock \bibinfo{publisher}{Chapman and Hall/CRC}.
	\bibitem[{Lau \& Lau(1996)}]{lau1996newsstand}
	\bibinfo{author}{Lau, H.-S.}, \& \bibinfo{author}{Lau, A. H.-L.}
	(\bibinfo{year}{1996}).
	\newblock \bibinfo{title}{The newsstand problem: {A} capacitated
		multiple-product single-period inventory problem}.
	\newblock {\it \bibinfo{journal}{European Journal of Operational Research}\/}
	\bibinfo{volume}{94}(\bibinfo{number}{1}):\bibinfo{pages}{29--42}.
	\bibitem[{Levi et~al.(2015)Levi, Perakis \& Uichanco}]{levi2015data}
	\bibinfo{author}{Levi, R.}, \bibinfo{author}{Perakis, G.}, \&
	\bibinfo{author}{Uichanco, J.} (\bibinfo{year}{2015}).
	\newblock \bibinfo{title}{The data-driven newsvendor problem: {N}ew bounds and
		insights}.
	\newblock {\it \bibinfo{journal}{Operations Research}\/}
	\bibinfo{volume}{63}(\bibinfo{number}{6}):\bibinfo{pages}{1294--1306}.
	\bibitem[{Levi et~al.(2007)Levi, Roundy \& Shmoys}]{levi2007provably}
	\bibinfo{author}{Levi, R.}, \bibinfo{author}{Roundy, R.~O.}, \&
	\bibinfo{author}{Shmoys, D.~B.} (\bibinfo{year}{2007}).
	\newblock \bibinfo{title}{Provably near-optimal sampling-based policies for
		stochastic inventory control models}.
	\newblock {\it \bibinfo{journal}{Mathematics of Operations Research}\/}
	\bibinfo{volume}{32}(\bibinfo{number}{4}):\bibinfo{pages}{821--839}.
	\bibitem[{Mackay et~al.(2019)Mackay, Vicol, Lorraine, Duvenaud \&
		Grosse}]{mackay2018selftuning}
	\bibinfo{author}{Mackay, M.}, \bibinfo{author}{Vicol, P.},
	\bibinfo{author}{Lorraine, J.}, \bibinfo{author}{Duvenaud, D.}, \&
	\bibinfo{author}{Grosse, R.} (\bibinfo{year}{2019}).
	\newblock \bibinfo{title}{Self-tuning networks: {B}ilevel optimization of
		hyperparameters using structured best-response functions}.
	\newblock In {\it \bibinfo{booktitle}{International Conference on Learning
			Representations (ICLR)}\/}.
	\newblock \bibinfo{address}{New Orleans, Louisiana, USA}.
	\bibitem[{Mandi et~al.(2020)Mandi, Demirovi{\'c}, Stuckey \&
		Guns}]{mandi2020smart}
	\bibinfo{author}{Mandi, J.}, \bibinfo{author}{Demirovi{\'c}, E.},
	\bibinfo{author}{Stuckey, P.~J.}, \& \bibinfo{author}{Guns, T.}
	(\bibinfo{year}{2020}).
	\newblock \bibinfo{title}{Smart predict-and-optimize for hard combinatorial
		optimization problems}.
	\newblock {\it \bibinfo{journal}{Proceedings of the AAAI Conference on
			Artificial Intelligence}\/}
	\bibinfo{volume}{34}(\bibinfo{number}{2}):\bibinfo{pages}{1603--1610}.
	\bibitem[{Molina et~al.(2002)Molina, Belanche \& Nebot}]{1183917}
	\bibinfo{author}{Molina, L.}, \bibinfo{author}{Belanche, L.}, \&
	\bibinfo{author}{Nebot, A.} (\bibinfo{year}{2002}).
	\newblock \bibinfo{title}{Feature selection algorithms: {A} survey and
		experimental evaluation}.
	\newblock In {\it \bibinfo{booktitle}{Proceedings of the IEEE International
			Conference on Data Mining (ICDM)}\/}.
	\newblock \bibinfo{address}{Maebashi City, Japan} (pp.
	\bibinfo{pages}{306--313}).
	\bibitem[{Moon \& Gallego(1994)}]{moon1994distribution}
	\bibinfo{author}{Moon, I.}, \& \bibinfo{author}{Gallego, G.}
	(\bibinfo{year}{1994}).
	\newblock \bibinfo{title}{Distribution free procedures for some inventory
		models}.
	\newblock {\it \bibinfo{journal}{Journal of the Operational Research
			Society}\/}
	\bibinfo{volume}{45}(\bibinfo{number}{6}):\bibinfo{pages}{651--658}.
	\bibitem[{Oroojlooyjadid et~al.(2020)Oroojlooyjadid, Snyder \&
		Tak{\'a}{\v{c}}}]{oroojlooyjadid2020applying}
	\bibinfo{author}{Oroojlooyjadid, A.}, \bibinfo{author}{Snyder, L.~V.}, \&
	\bibinfo{author}{Tak{\'a}{\v{c}}, M.} (\bibinfo{year}{2020}).
	\newblock \bibinfo{title}{Applying deep learning to the newsvendor problem}.
	\newblock {\it \bibinfo{journal}{IISE Transactions}\/}
	\bibinfo{volume}{52}(\bibinfo{number}{4}):\bibinfo{pages}{444--463}.
	\bibitem[{Perakis \& Roels(2008)}]{perakis2008regret}
	\bibinfo{author}{Perakis, G.}, \& \bibinfo{author}{Roels, G.}
	(\bibinfo{year}{2008}).
	\newblock \bibinfo{title}{Regret in the newsvendor model with partial
		information}.
	\newblock {\it \bibinfo{journal}{Operations Research}\/}
	\bibinfo{volume}{56}(\bibinfo{number}{1}):\bibinfo{pages}{188--203}.
	\bibitem[{Petruzzi \& Dada(1999)}]{petruzzi1999pricing}
	\bibinfo{author}{Petruzzi, N.~C.}, \& \bibinfo{author}{Dada, M.}
	(\bibinfo{year}{1999}).
	\newblock \bibinfo{title}{Pricing and the newsvendor problem: {A} review with
		extensions}.
	\newblock {\it \bibinfo{journal}{Operations Research}\/}
	\bibinfo{volume}{47}(\bibinfo{number}{2}):\bibinfo{pages}{183--194}.
	\bibitem[{Qin et~al.(2011)Qin, Wang, Vakharia, Chen \& Seref}]{QIN2011361}
	\bibinfo{author}{Qin, Y.}, \bibinfo{author}{Wang, R.},
	\bibinfo{author}{Vakharia, A.~J.}, \bibinfo{author}{Chen, Y.}, \&
	\bibinfo{author}{Seref, M.~M.} (\bibinfo{year}{2011}).
	\newblock \bibinfo{title}{The newsvendor problem: {R}eview and directions for
		future research}.
	\newblock {\it \bibinfo{journal}{European Journal of Operational Research}\/}
	\bibinfo{volume}{213}(\bibinfo{number}{2}):\bibinfo{pages}{361--374}.
	\bibitem[{Sachs \& Minner(2014)}]{SACHS201428}
	\bibinfo{author}{Sachs, A.-L.}, \& \bibinfo{author}{Minner, S.}
	(\bibinfo{year}{2014}).
	\newblock \bibinfo{title}{The data-driven newsvendor with censored demand
		observations}.
	\newblock {\it \bibinfo{journal}{International Journal of Production
			Economics}\/} \bibinfo{volume}{149}:\bibinfo{pages}{28--36}.
	\bibitem[{Scarf(1958)}]{scarf1958min}
	\bibinfo{author}{Scarf, H.} (\bibinfo{year}{1958}).
	\newblock \bibinfo{title}{A min-max solution of an inventory problem}.
	\newblock In \bibinfo{editor}{K.~Arrow}, \bibinfo{editor}{S.~Karlin}, \&
	\bibinfo{editor}{H.~Scarf} (Eds.), {\it \bibinfo{booktitle}{Studies in the
			Mathematical Theory of Inventory and Production}\/}.
	\newblock \bibinfo{publisher}{Stanford University Press, Stanford} (pp.
	\bibinfo{pages}{201--209}).
	\bibitem[{See \& Sim(2010)}]{see2010robust}
	\bibinfo{author}{See, C.-T.}, \& \bibinfo{author}{Sim, M.}
	(\bibinfo{year}{2010}).
	\newblock \bibinfo{title}{Robust approximation to multiperiod inventory
		management}.
	\newblock {\it \bibinfo{journal}{Operations Research}\/}
	\bibinfo{volume}{58}(\bibinfo{number}{3}):\bibinfo{pages}{583--594}.
	\bibitem[{Shapiro(2003)}]{SHAPIRO2003353}
	\bibinfo{author}{Shapiro, A.} (\bibinfo{year}{2003}).
	\newblock \bibinfo{title}{Monte carlo sampling methods}.
	\newblock In \bibinfo{editor}{A.~Ruszczynski}, \& \bibinfo{editor}{A.~Shapiro}
	(Eds.), {\it \bibinfo{booktitle}{Stochastic Programming}\/},
	volume~\bibinfo{volume}{10} of {\it \bibinfo{series}{Handbooks in Operations
			Research and Management Science}\/}.
	\newblock \bibinfo{publisher}{Elsevier, Amsterdam} (pp.
	\bibinfo{pages}{353--425}).
	\bibitem[{Wang \& Webster(2009)}]{wang2009loss}
	\bibinfo{author}{Wang, C.~X.}, \& \bibinfo{author}{Webster, S.}
	(\bibinfo{year}{2009}).
	\newblock \bibinfo{title}{The loss-averse newsvendor problem}.
	\newblock {\it \bibinfo{journal}{Omega}\/}
	\bibinfo{volume}{37}(\bibinfo{number}{1}):\bibinfo{pages}{93--105}.
	\bibitem[{Wang et~al.(2016)Wang, Glynn \& Ye}]{wang2016likelihood}
	\bibinfo{author}{Wang, Z.}, \bibinfo{author}{Glynn, P.~W.}, \&
	\bibinfo{author}{Ye, Y.} (\bibinfo{year}{2016}).
	\newblock \bibinfo{title}{Likelihood robust optimization for data-driven
		problems}.
	\newblock {\it \bibinfo{journal}{Computational Management Science}\/}
	\bibinfo{volume}{13}(\bibinfo{number}{2}):\bibinfo{pages}{241--261}.
	\bibitem[{Yue et~al.(2006)Yue, Chen \& Wang}]{yue2006expected}
	\bibinfo{author}{Yue, J.}, \bibinfo{author}{Chen, B.}, \&
	\bibinfo{author}{Wang, M.-C.} (\bibinfo{year}{2006}).
	\newblock \bibinfo{title}{Expected value of distribution information for the
		newsvendor problem}.
	\newblock {\it \bibinfo{journal}{Operations Research}\/}
	\bibinfo{volume}{54}(\bibinfo{number}{6}):\bibinfo{pages}{1128--1136}.
	\bibitem[{Zhang \& Gao(2017)}]{zhang2017assessing}
	\bibinfo{author}{Zhang, Y.}, \& \bibinfo{author}{Gao, J.}
	(\bibinfo{year}{2017}).
	\newblock \bibinfo{title}{Assessing the performance of deep learning algorithms
		for newsvendor problem}.
	\newblock In \bibinfo{editor}{D.~Liu}, \bibinfo{editor}{S.~Xie},
	\bibinfo{editor}{Y.~Li}, \bibinfo{editor}{D.~Zhao}, \&
	\bibinfo{editor}{E.-S.~M. El-Alfy} (Eds.), {\it \bibinfo{booktitle}{Neural
			Information Processing}\/}.
	\newblock \bibinfo{publisher}{Springer, Cham} (pp. \bibinfo{pages}{912--921}).
	\bibitem[{Zhu et~al.(2012)Zhu, Huang \& Li}]{zhu2012semiparametric}
	\bibinfo{author}{Zhu, L.}, \bibinfo{author}{Huang, M.}, \& \bibinfo{author}{Li,
		R.} (\bibinfo{year}{2012}).
	\newblock \bibinfo{title}{Semiparametric quantile regression with
		high-dimensional covariates}.
	\newblock {\it \bibinfo{journal}{Statistica Sinica}\/}
	\bibinfo{volume}{22}(\bibinfo{number}{4}):\bibinfo{pages}{1379--1401}.
	
\end{thebibliography}
\end{document}